\documentclass[11pt]{article}
\usepackage[margin=1in]{geometry}
\usepackage{graphicx} % Required for inserting images

\usepackage[T1]{fontenc}
\usepackage[colorlinks=true,linkcolor=red,citecolor=Green]{hyperref}

\usepackage{url}            
\usepackage{booktabs}       
\usepackage{amsfonts}       
\usepackage{nicefrac}       
\usepackage{microtype}     
\usepackage{natbib}

\usepackage{comment}
\usepackage{amsmath,enumitem,amsthm,float}
\usepackage{mathrsfs}
\usepackage{graphicx}
\usepackage{rotating}

\usepackage{amssymb}
\usepackage{autobreak}
\usepackage{mathtools}
\usepackage[dvipsnames]{xcolor}
\usepackage{bbm}

\usepackage{multirow}
\usepackage{framed}
\usepackage{bigints}

\usepackage[linesnumbered,ruled,lined]{algorithm2e}

\usepackage{newfloat}
\usepackage{caption}
\usepackage[most]{tcolorbox}
\usepackage[capitalise,noabbrev]{cleveref}

\DeclareFloatingEnvironment[
  fileext=los,
  listname={List of Subproblems},
  name=Subproblem,
  placement=!htbp
]{subproblem}

\tcbset{
  subproblemstyle/.style={
    colback=white,
    colframe=black,
    fonttitle=\bfseries,        % title font
    coltitle=black,             % title color
    colbacktitle=white,         % title background color
    toptitle=1mm,               % spacing above title
    bottomtitle=1mm,            % spacing below title
    titlerule=0.4pt,            % thin line under the title
    titlerule style=black,      % color of the line
    boxrule=0.5pt,
    arc=2mm,
    outer arc=2mm,
    boxsep=2pt,
    left=6pt, right=6pt, top=4pt, bottom=4pt,
  }
}

\crefname{subproblem}{Subproblem}{Subproblems}
\Crefname{subproblem}{Subproblem}{Subproblems}

\newcommand{\KwInput}[1]{\textbf{Input:} #1\\}
\newcommand{\KwOutput}[1]{\textbf{Output:} #1\\}

\newtheorem{theorem}{Theorem}[section]
\newtheorem{lemma}[theorem]{Lemma}
\newtheorem{definition}[theorem]{Definition}

\newtheorem{proposition}[theorem]{Proposition}
\newtheorem{corollary}[theorem]{Corollary}

\newcommand{\scr}{\mathscr}

\DeclareMathOperator*{\argmax}{argmax}

\newcommand{\Ascr}{\mathscr{A}}

\newcommand{\wt}{\widetilde}

\renewcommand{\H}{\mathcal{H}}
\newcommand{\cH}{\mathcal{H}}

\newcommand{\bt}{\boldsymbol{t}}

\newcommand{\bx}{\boldsymbol{x}}

\newcommand{\bs}{\boldsymbol{s}}

\newcommand{\N}{\mathbb{N}}
\newcommand{\D}{\mathcal{D}}

\newcommand{\G}{\mathcal{G}}

\newcommand{\Log}{\operatorname{Log}}

\newcommand{\M}{\mathcal{M}}

\newcommand{\DS}{\mathrm{DS}}
\newcommand{\n}{\mathrm{N}}

\newcommand{\opt}{h^{\star}}

\renewcommand{\deg}{\mathsf{deg}}
\renewcommand{\d}{\mathsf{d}}
\newcommand{\outdeg}{\mathsf{outdeg}}

\newcommand{\RE}{\mathrm{RE}}

\newcommand{\Maj}{\mathrm{Maj}}
\newcommand{\A}{\mathcal{A}}

\newcommand{\br}{\boldsymbol{r}}

\newcommand{\unif}{\mathrm{Unif}}

\renewcommand{\epsilon}{\varepsilon}
\newcommand{\eps}{\varepsilon}

\newcommand{\X}{\mathcal{X}}
\newcommand{\Y}{\mathcal{Y}}

\newcommand{\Z}{\mathcal{Z}}
\newcommand{\er}{\mathrm{er}}

\newcommand{\ind}{\boldsymbol{1}}

\newcommand{\F}{\mathcal{F}}

\renewcommand{\P}{{\mathbb{P}}}

\newcommand{\R}{\mathbb{R}}
\newcommand{\E}{{{\mathbb{E}}}}

\newcommand{\power}{\mathbbm{2}}

\newcommand{\size}{\mathsf{k}}
\newcommand{\tsum}{\textstyle\sum}
\newcommand{\oig}{\mathrm{OIG}}
\newcommand{\supp}{\mathrm{supp}}
\newcommand{\MW}{\mathsf{MW}}
\newcommand{\AG}{\mathrm{AG}}

\newcommand{\md}{\mathsf{md}}
\newcommand{\avgdeg}{\mathsf{avgdeg}}
\newcommand{\avgoutdeg}{\mathsf{avgoutdeg}}

\newcommand{\learner}{\mathsf{MAPL}}
\newcommand{\dir}{\mathrm{dir}}

\newcommand\nnfootnote[2]{%
  \begin{NoHyper}
  \renewcommand\thefootnote{#1}\footnotetext{#2}%
  \end{NoHyper}
}

\title{Sample Complexity of Agnostic Multiclass Classification:\\ 
Natarajan Dimension Strikes Back }
\author{ }
\date{}

\begin{document}

\author{
Alon Cohen$^{*,\dag}$
\and
Liad Erez$^{*}$
\and
Steve Hanneke$^{\S}$
\and
Tomer Koren$^{*,\dag}$
\and
Yishay Mansour$^{*,\dag}$
\and
Shay Moran$^{\ddag,\dag}$
\and
Qian Zhang$^{\S}$
}

\maketitle

\nnfootnote{*}{Tel Aviv University, Tel Aviv, Israel}
% \nnfootnote{*}{ Blavatnik School of Computer Science and AI, Tel Aviv University, Tel Aviv, Israel.}
\nnfootnote{\textdagger}{ Google Research Tel Aviv, Israel}
\nnfootnote{\S}{Purdue University, Indiana, USA}
\nnfootnote{\textdaggerdbl}{Technion---Israel Institute of Technology, Haifa, Israel}

\begin{abstract}
The fundamental theorem of statistical learning establishes that binary PAC learning is governed by a single parameter---the Vapnik-Chervonenkis ($\mathtt{VC}$) dimension---which controls both learnability and sample complexity. Extending this characterization to multiclass classification has long been challenging, since the early work of Natarajan in the late 80's that proposed the Natarajan dimension ($\mathtt{Nat}$) as a natural analogue of the VC dimension. 
Daniely and Shalev-Shwartz (2014) introduced the $\mathtt{DS}$ dimension, later shown by Brukhim et al.\ (2022) to characterize multiclass \emph{learnability}.  
Brukhim et al.\ (2022) also demonstrated that the Natarajan and $\mathtt{DS}$ dimensions can diverge arbitrarily, so that multiclass learning appears to be governed by $\mathtt{DS}$ rather than $\mathtt{Nat}$.
We show that the agnostic multiclass PAC sample complexity is in fact governed by \emph{two distinct dimensions}. Specifically, we prove nearly tight agnostic sample complexity bounds that, up to logarithmic factors, take the form
$$
\frac{\mathtt{DS}^{1.5}}{\epsilon} + \frac{\mathtt{Nat}}{\epsilon^2}
$$
where $\epsilon$ is the excess risk. This bound is tight up to a $\sqrt{\mathtt{DS}}$ factor in the first lower-order term, nearly matching known $\mathtt{Nat}/\epsilon^2$ and $\mathtt{DS}/\epsilon$ lower bounds. 
The first term reflects the DS-controlled regime, while the second reveals that the Natarajan dimension still dictates asymptotic behavior for small $\epsilon$. Thus, unlike in binary or online classification---where a single dimension (VC or Littlestone) controls both phenomena---multiclass learning inherently involves \emph{two structural parameters}.

Our technical approach departs significantly from traditional agnostic learning methods based on uniform convergence or reductions-to-realizable techniques. A key ingredient is a novel online procedure, based on a self-adaptive multiplicative-weights algorithm which performs a label-space reduction. This approach may be of independent interest and find further applications.
\end{abstract}

\section{Introduction}
Most theoretical work in PAC learning has focused on the binary case, yet many natural learning problems are multiclass, with a label space that can be vast or continually expanding.  
In recommender systems, for instance, new products or restaurants are constantly added, and in language modeling, new words emerge over time—so the label space is never truly fixed. Such settings challenge existing theory and reveal phenomena absent from the binary case,\footnote{
For instance, empirical risk minimization (ERM)—which underlies much of binary PAC theory—can fail to learn certain multiclass hypothesis classes; this breakdown parallels the departure from ERM principles often observed in practical learning algorithms.} making multiclass learning both practically relevant and theoretically compelling.

In the binary classification setting, a cornerstone of statistical learning theory is the \emph{fundamental theorem of PAC learning}, which links learnability to a single combinatorial parameter—the Vapnik--Chervonenkis (VC) dimension~\cite{Vapnik68}.  
For a hypothesis class $\mathcal{H} \subseteq \{0,1\}^\mathcal{X}$, the VC dimension is the largest integer $d$ for which there exist points $x_1, \dots, x_d \in \mathcal{X}$ that can be labeled in all $2^d$ possible ways by functions in $\mathcal{H}$.  
The theorem asserts that a class is PAC learnable if and only if its VC dimension is finite.  
Moreover, the sample complexity of learning $\mathcal{H}$ 
with error $\epsilon$ is $\Theta(\mathtt{VC}/\epsilon)$ in the realizable case and is $\Theta(\mathtt{VC}/\epsilon^2)$ in the agnostic case.\footnote{The agnostic sample complexity is defined as the smallest sample size $m$ such that there exists an algorithm guaranteeing misclassification error rate at most $\epsilon$-larger than the best $h^\star \in \mathcal{H}$, with probability at least $1-\delta$, when given as input  
%over the draw of 
$m$ i.i.d.\ examples from an unknown arbitrary distribution over $\mathcal{X} \times \mathcal{Y}$.  The realizable case restricts to the special case where $h^\star$ has error rate zero.  See Section~\ref{sec:definitions-of-sample-complexity} for formal definitions.}
%can also be expressed purely in terms of the VC dimension with excess error~$\epsilon$ is $\Theta(\mathtt{VC}/\epsilon^2)$.

A natural question is whether there exists an analogue of the fundamental theorem of PAC learning in the multiclass setting.  
That is:
\begin{center}
\emph{Can multiclass learnability be characterized through a single combinatorial dimension,\\
together with corresponding sample complexity bounds?}
\end{center}
This problem has attracted considerable attention since the late 1980s, beginning with the pioneering work of Natarajan and Tadepalli~\citep{natarajan1988two,natarajan:89}.  

Natarajan introduced a natural multiclass analogue of the VC dimension, now called the \emph{Natarajan dimension}.  
Given a hypothesis class $\mathcal{H} \subseteq \mathcal{Y}^\mathcal{X}$, a set of points $x_1, \dots, x_d \in \mathcal{X}$ is said to be \emph{Natarajan-shattered} by $\mathcal{H}$ if there exist two labelings 
$f, g : \{x_1, \dots, x_d\} \to \mathcal{Y}$ such that $f(x_i) \neq g(x_i)$ for all~$i$, and for every subset $S \subseteq \{1, \dots, d\}$ there exists $h \in \mathcal{H}$ that agrees with~$f$ on $x_i$, for~$i\in S$ and with~$g$ on $x_j$ for $j\notin S$.  
The Natarajan dimension of~$\mathcal{H}$ is the largest~$d$ for which such a set exists.  
Note that in the binary setting ($\mathcal{Y}=\{0,1\}$), the Natarajan dimension coincides with the VC dimension.  
Furthermore, it provides a lower bound of order~$\mathtt{Nat}/\epsilon^2$ on the sample complexity of agnostic PAC learning, following the standard ``no-free-lunch'' argument that also yields the~$\mathtt{VC}/\epsilon^2$ lower bound in the binary case.

Alongside the Natarajan dimension, Natarajan also considered another combinatorial quantity, later called the \emph{Graph dimension}, which upper bounds the sample complexity by $\mathtt{Graph}/\epsilon^2$.  
The Graph dimension exactly characterizes \emph{uniform convergence} of empirical and true errors—that is, it captures when the empirical error uniformly approximates the true error across all hypotheses in the class.  
In such cases, \emph{empirical risk minimization (ERM)}, which simply chooses a hypothesis minimizing the empirical error, achieves low generalization error and hence guarantees learnability.  
However, Natarajan already observed that there exist \emph{learnable} classes with infinite Graph dimension but finite Natarajan dimension.
These examples demonstrate that uniform convergence and ERM do not suffice to explain multiclass PAC learning, in sharp contrast to the binary case.
At the same time, he showed that when the number of labels is bounded, learnability is completely characterized by the Natarajan and Graph dimensions—both are finite if and only if the class is learnable.
The central open question, therefore, was whether the Natarajan dimension continues to characterize multiclass PAC learnability once the label space is unbounded.

Building on Natarajan’s ideas, subsequent works by Ben-David, Cesa-Bianchi, Haussler, and Long \citep{ben-david:95} and by Haussler and Long \citep{haussler:95} proposed general frameworks for defining combinatorial dimensions in the multiclass setting.
These works clarified and extended Natarajan’s characterization when the number of labels is bounded, but the general unbounded case remained elusive.  
A later line of research investigated broader principles governing multiclass learning~\citep{NIPS2006_a11ce019,DanielySS12,daniely:15,daniely2014optimal,daniely2015inapproximability}.  
These studies revealed a sharp contrast between the bounded and unbounded regimes: in particular, they showed that not only can the classical ERM principle fail (as already observed by Natarajan), but even significantly more general variants of ERM break down in the unbounded-label setting.  
Notably, \cite{daniely2014optimal} proved that certain multiclass hypothesis classes are \emph{not even properly learnable}—that is, any learner must output hypotheses outside the class in order to succeed.  
This breakdown of ERM-based approaches, already hinted at by Natarajan’s examples, pointed to a fundamental gap in our understanding of multiclass learnability.

A major step forward came in 2014, when Daniely and Shalev-Shwartz introduced a new dimension in the context of \emph{one-inclusion graph} algorithms—a universal family of transductive learning rules capturing the combinatorial structure underlying multiclass learnability~\citep{daniely2014optimal}.  
This new quantity, now called the \emph{Daniely--Shalev-Shwartz (DS) dimension}, plays a central role in the modern theory.  

Formally, the DS dimension of a hypothesis class $\mathcal{H} \subseteq \mathcal{Y}^\mathcal{X}$ is the largest integer $d$ such that there exists a set of points $\{x_1,\ldots,x_d\} \subseteq \X$ for which the set $\{(h(x_1),\dots,h(x_d)) : h \in \mathcal{H}\}$ contains a finite \emph{pseudo-cube}.
A subset $B \subseteq \mathcal{Y}^d$ is called a pseudo-cube if for every $b \in B$ and every coordinate $i \le d$ there exists $b' \in B$ such that $b'(i) \ne b(i)$ and $b'(j)=b(j)$ for all $j \ne i$. In short: \emph{every vector has a neighbor in every direction.}  
In the binary case, pseudo-cubes coincide exactly with standard cubes: $B \subseteq \{0,1\}^d$ is a pseudo-cube if and only if $B = \{0,1\}^d$.  
Similarly, in the multiclass case, Natarajan shattering corresponds to pseudo-cubes that are isomorphic to binary cubes, but there also exist far richer pseudo-cubes whose structure is more intricate (see~\citep{brukhim2022characterization} for examples).

Daniely and Shalev-Shwartz showed that the DS dimension provides a lower bound of order $\mathtt{DS}/\epsilon$ on the sample complexity, even in the realizable case where the best hypothesis in the class achieves zero error.  
Later, Brukhim, Carmon, Dinur, Moran, and Yehudayoff~\citep{brukhim2022characterization} proved that the DS dimension characterizes multiclass PAC learnability: a class is learnable if and only if it has finite DS dimension.  
They further established quantitative upper bounds of roughly $\mathtt{DS}^{1.5}/\epsilon$ in the realizable case and $\mathtt{DS}^{1.5}/\epsilon^2$ in the agnostic case (up to logarithmic factors).
% , and showed that the DS dimension can be arbitrarily larger than the Natarajan dimension, implying that no bound depending only on the Natarajan dimension can hold in full generality.

Taken together, these results suggested that the DS dimension is the fundamental quantity governing multiclass learning, both qualitatively and quantitatively.  
The Natarajan dimension, by contrast, appeared to have lost its relevance—indeed, Brukhim et al.\ constructed classes with Natarajan dimension $1$ that are nevertheless not learnable.  
This viewpoint naturally leads to the conjecture that the optimal sample complexity of multiclass PAC learning scales as $\mathtt{DS}/\epsilon^2$.  

\paragraph{Main Result.}
In this work, we show that, perhaps surprisingly, the Natarajan dimension remains significant: it determines the leading asymptotic term in the agnostic sample complexity. 
%The sample complexity is defined as the smallest $m$ for which this is achievable by some algorithm.
%(without computational restrictions).
To state our main result, we introduce the \emph{realizable dimension}.  
For a concept class~$\mathcal{H}$, let~$d_{\mathrm{real}}(\mathcal{H})$ denote the optimal number of examples required to achieve risk at most a constant $c < 1/2$ in the realizable case.\footnote{The specific choice of $c$ is not important—any constant $<\tfrac{1}{2}$ would lead to the same guarantees in our main result, up to log factors.}
%\qian{will there be a log factor after boosting? Answer: you mean $\log 1/\eps$ factor? If $\eps$ is contant it is also a constant}\qian{I mean $\log d_{\mathrm{real}}$ factors. Answer: I don't see why there should be such a factor - it is easiest to see by a minimax argument (boosting should give the same bounds). Am I missing something?}
%\qian{My understanding for reduction is that we can construct a sample compression scheme of size $d_{\mathrm{real}}\log(n)$ so that the error rate is bounded by $d_{\mathrm{real}}\log^2(n)/n$. Then for constant upper bound we still need $\Omega(d_{\mathrm{real}}\log^2(d_{\mathrm{real}}))$ samples. Maybe you are referring to a better approach? Answer: no, you are correct!}
%
In the binary-labeled case,~$d_{\mathrm{real}}(\mathcal{H})$ is proportional to the VC dimension, while in the multiclass setting it lies between~$\Omega(\mathtt{DS})$ and~$\widetilde O(\mathtt{DS}^{1.5})$~\cite{daniely2014optimal,brukhim2022characterization}.

\begin{theorem}[Main]
\label{thm:main-informal}
Let $\mathcal{H} \subseteq \mathcal{Y}^\mathcal{X}$ be a concept class, and let $\M_{\H}^{\AG}(\epsilon,\delta)$ denote the sample complexity of 
\emph{agnostic} PAC learning~$\mathcal{H}$ 
with error~$\epsilon$ and confidence~$1-\delta$.  
Then,
\[
\M_{\H}^{\AG}(\epsilon,\delta) \;=\;
\widetilde\Theta\!\left(
\frac{\mathtt{Nat}+\log(1/\delta)}{\epsilon^2}
+\frac{d_{\mathrm{real}}}{\epsilon}
\right),
\]
where the $\widetilde\Theta$ notation hides logarithmic factors in~$\mathtt{Nat}$,~$d_{\mathrm{real}}$, and~$1/\epsilon$.
A detailed and fully quantified version of this result appears in Theorem~\ref{thm:main}.
\end{theorem}

Since our sample complexity bound is tightly expressed in terms of the realizable dimension, any improvement in the known upper and lower sample complexity bounds in the realizable case would directly imply, through our approach, analogously improved bounds in the more general agnostic case. 
As noted above, the best bounds for the realizable dimension are currently $\Omega(\mathtt{DS})$ and~$\widetilde O(\mathtt{DS}^{1.5})$, which implies that the agnostic multiclass sample complexity is bounded by
\[
\widetilde O\!\left(
\frac{\mathtt{Nat}+\log(1/\delta)}{\epsilon^2}
+\frac{\mathtt{DS}^{1.5}}{\epsilon}
\right),
\]
and this bound is tight up to a $\sqrt{\mathtt{DS}}$ factor in the lower-order $1/\epsilon$ term.

The appearance of the realizable dimension in our bound also hints at the structure of our algorithm.  
It proceeds in two stages.  
In the first stage, we invoke an optimal realizable-case learner to construct a short \emph{list predictor}, effectively reducing the potentially unbounded label space to a bounded one.  
The sample complexity of this step corresponds to the second summand $\widetilde{\Theta}(d_{\mathrm{real}}/\epsilon)$ in the bound above.  
In the second stage, we learn the induced class over the reduced label space and apply Natarajan-dimension upper bounds, which hold when the label space is bounded.  
The first stage---label-space reduction---relies on a novel multiplicative-weights-based learning procedure, inspired by online learning methods.  
We believe this component is of independent interest and may find further applications.  
The next section provides an overview of the algorithm and its analysis, focusing on these novel and nontrivial ingredients.

\subsection{Technical Overview}
The goal of this section is to outline the main algorithmic contribution of the paper and to highlight the central technical challenges—particularly the parts that depart from standard approaches and require new ideas. We begin with a bird’s-eye view of the overall learning algorithm, describing its structure and the roles played by its components. We then gradually zoom in on the difficulties that arise in each stage and explain how these are addressed in our framework.

\begin{itemize}
    \item[(i)] \emph{Label-space reduction.}  
    Using $\widetilde O(d_{\mathrm{real}}/\varepsilon)$ samples (the realizable-rate sample complexity), we learn a short list of classifiers $\{h_1, \ldots, h_T\}$ that together specify, for each instance $x$, a small set of candidate labels $\{h_1(x), \ldots, h_T(x)\}$.  
    This ensures that although the original label space may be extremely large or even infinite, the \emph{effective} label space---the set of labels we ever need to consider for each $x$---has bounded size. This component will be further split into two subproblems, the first of which computes an improper finite cover of $\H$, denoted by $\F$, and a second subproblem which encapsulates the main technical challenge of this paper and aims to learn the short list $\{ h_1,\ldots,h_T\} \subseteq \F$ which achieves the desired reduction in label-space size.

    \item[(ii)] \emph{Learning a list-bounded classifier.}  
    With the label space now reduced, we use $\widetilde O(\mathtt{Nat}/\varepsilon^2)$ samples to learn a classifier $h$ that, for each $x$, predicts only among the candidate labels $\{h_1(x),\ldots,h_T(x)\}$, yet still competes with the best hypothesis in the original class.
\end{itemize}

The same two-step structure was also used by \cite{brukhim2022characterization} for the realizable case.  
Their algorithm first learns a short list of classifiers \(h_1,\dots,h_T\) such that \(Y\in\{h_1(X),\dots,h_T(X)\}\) for typical $(X,Y)\sim P$, and then learns a final classifier whose predictions are restricted to this list.  
(Here and throughout the technical overview, we use \(\text{list}(X)\) as shorthand for \(\{h_1(X),\dots,h_T(X)\}\).)

A natural idea for extending this approach to the agnostic setting is to require the learned list to satisfy
\begin{align}
\label{eq:list-wrong-objective}
\P(\,Y \notin \text{list}(X)\,) 
\;\le\; \er_P(h^\star) + \epsilon,
\end{align}
where \(h^\star\) is the optimal hypothesis in \(\mathcal{H}\).
However, this guarantee is \emph{insufficient} for the success of the second stage. 
To see why this is the case, consider a multiclass problem over $K+1$ labels~$\{0,1,\ldots,K \}$ and the hypothesis class consisting of the constant classifiers $\H = \{h_0,\ldots,h_{K} \}$ where~$h_i$ classifies every point with the label $i$. 
The data distribution sets the label $Y$ equal to $0$ with prob.\ $1/3$ and otherwise uniform over $\{1,\ldots,K\}$.
Consider the list $\{ h_1,\ldots,h_K \}$ which has an expected error of $1/3$, while the best hypothesis $h_0$ has an expected error of $2/3$. 
The list thus clearly satisfies \cref{eq:list-wrong-objective}. 
However, any classifier constrained to be consistent with the list has an expected  error of $1-O(1/K)$, which is far worse than the best possible in $\H$. 
For a more detailed analysis, see \cref{sec:example}.

Instead of \cref{eq:list-wrong-objective}, what \emph{suffices} for the second stage is a stronger requirement: for every~\(h \in \mathcal{H}\), with high probability over $(X,Y)\sim P$,
\begin{align}
h(X)=Y \;\Longrightarrow\; h(X)\in \text{list}(X).
\end{align}
In words, this means that with high probability, whenever \(h\) predicts the correct label, that label also appears in the learned list.  
\emph{This condition ensures that the optimal hypothesis \(h^\star\) is captured by the list on the points where it is correct, allowing the second stage to compete with \(h^\star\) using only the reduced label space.}  
Achieving this guarantee while keeping the list short \emph{and} using only realizable-rate sample complexity is the central technical challenge of the paper and is discussed in further detail in Section~\ref{sec:label-space-reduction} below.

\paragraph{Algorithm overview.}
We now turn to a more detailed description of the algorithm.  
It is instructive to decompose it into three conceptual stages, each addressing a distinct modular subproblem.  
The output of each subproblem serves as the input to the next, creating a clean pipeline that mirrors the structure of the overall learner.  
The first two subproblems capture the \emph{label-space reduction} component of the algorithm, while the third subproblem formalizes the task of \emph{learning a list-bounded classifier}.  
This separation allows us to highlight the contribution of each stage and to isolate the technical challenges.

\begin{subproblem}[H]
% \caption{Approximation by improper finite cover}
\label{subproblem:finite-class}
\begin{tcolorbox}[subproblemstyle, title={Subproblem 1: Approximation by improper finite cover}]
\KwInput{A hypothesis class $\H$ of realizable dimension $d$, parameters $\eps,\delta \in (0,1)$, and a sample $S_1 \sim P^{n_1}$ of size 
$n_1 = \widetilde O\bigl( \frac{d + \log(1/\delta)}{\eps} \bigr)$.}
\KwOutput{A finite set $\F \subseteq \Y^\X$ of size $n_1^{O(d \log n_1)}$ such that for every $h^\star \in \H$, with probability at least $1-\delta$, there exists 
$f^\star \in \F$ satisfying
\[
\P\bigl(h^\star(X)=Y \ \land\ h^\star(X)\neq f^\star(X)\bigr) \le \eps.
\]
\vspace{-15pt}
}
\noindent
This step guarantees that any $h^\star \in \H$ has a proxy $f^\star \in \F$ that agrees with it on almost all correctly labeled examples.

\end{tcolorbox}

\end{subproblem}

%-------------------------------
\begin{subproblem}[H]
% \caption{Learning a list}
\label{subproblem:label-space-reduction}
\begin{tcolorbox}[subproblemstyle,title={Subproblem 2: Learning a list}]

\KwInput{A finite class $\F$, parameters $\eps,\delta \in (0,1)$, and a sample $S_2 \sim P^{n_2}$ of size 
$n_2 = O\bigl(\frac{1}{\eps}\log\frac{|\F|}{\delta}\bigr)$.}
\KwOutput{A list $\{h_1,\dots,h_{n_2}\} \subseteq \F$ such that for every $f^\star \in \F$, with probability at least $1-\delta$,
\[
\P\bigl(f^\star(X)=Y \ \land\ f^\star(X)\notin\{h_1(X),\dots,h_{n_2}(X)\}\bigr) \le \eps.
\]
\vspace{-15pt}
}
\noindent This step performs a label-space reduction: almost all correct predictions of any $f^\star \in \mathcal{F}$ fall within the list $\{h_1,\dots,h_{n_2}\}$, whose size is exponentially smaller than $\lvert \mathcal{F} \rvert$.
\end{tcolorbox}

\end{subproblem}

%-------------------------------
\begin{subproblem}[H]
% \caption{Learning a list-bounded classifier}
\label{subproblem:classifier-from-list}
\begin{tcolorbox}[subproblemstyle,title={Subproblem 3: Learning a list-bounded classifier}]

\KwInput{A list $\{h_1,\dots,h_T\}$, a hypothesis class $\H$ with Natarajan dimension $d_N$, parameters $\eps,\delta \in (0,1)$, and a sample
$S_3 \sim P^{n_3}$ of size 
\(
n_3 = \widetilde O\bigl(\frac{d_N \log T + \log(1/\delta)}{\eps^2}\bigr).
\)
}
\KwOutput{A predictor $h$ satisfying $h(x)\in\{h_1(x),\dots,h_T(x)\}$ for all $x$, such that for every $h^\star \in \H$, with probability at least $1-\delta$,
\[
\P(h(X) = Y)
\ge 
\P\bigl(h^\star(X)=Y \ \land\ h^\star(X)\in\{h_1(X),\dots,h_T(X)\}\bigr)
- \eps.
\]
\vspace{-15pt}
}
\noindent
This stage performs list-bounded agnostic learning: we compete with the best $h^\star \in \H$, but correct predictions of $h^\star$ outside the list are treated as errors.  
Thus the learner only needs to operate over the restricted label space induced by the list, essentially reducing the problem to multiclass learning with a bounded number of labels.
\end{tcolorbox}

\end{subproblem}

%\medskip
% \textcolor{red}{The name of the first stage may be slightly misleading: the finite-class reduction primarily serves as preprocessing for the second stage, where finiteness is actually essential. In the final stage we return to reasoning about the full (infinite) class.}
% \textcolor{red}{Perhaps a better interpretation is that the first step reduces the \emph{effective} label space to a finite --- though still exponentially large --- set by constructing a suitable finite cover.}

% \textcolor{red}{Somewhere early in the overview, we should briefly explain why it is justified to treat the case where each $x$ has at most $k<\infty$ possible labels as if the global label space were of size $k$. This intuition is not difficult to convey, and it is conceptually important for understanding our overall approach.}
% \textcolor{blue}{[Liad: How about calling it "finite hypothesis class approximation"?]}
%-------------------------------
\paragraph{Combining the three stages.}

Suppose all three subproblems succeed (which occurs with probability at least $1-3\delta$).  
Let $h$ be the final classifier returned by Subproblem 3, let $h^\star \in \H$ be an optimal hypothesis, and let 
$f^\star \in \F$ be the proxy guaranteed by Subproblem 1.  
The guarantees of the three subproblems give the following chain:
\begin{align*}
    \P \left(h^\star(X)=Y\right)
    &\le
    \P\left(h^\star(X)=Y \ \land\ f^\star(X)=Y\right) + \eps
    &&\text{(Subproblem 1)}\\[2mm]
    &\le
    \P\left(h^\star(X)=Y \ \land\ f^\star(X)=Y \ \land\ f^\star(X) \in \text{list}(X) \right) + 2\eps
    &&\text{(Subproblem 2)}\\[2mm]
    &=\P\left(h^\star(X)=Y \ \land\ f^\star(X)=Y \ \land\ h^\star(X) \in \text{list}(X) \right) + 2\eps\\
    &\le
    \P\left(h^\star(X)=Y \ \land\ h^\star(X) \in \text{list}(X) \right) + 2\eps \\
    &\le
    \P(h(X)=Y) + 3 \eps.
    &&\text{(Subproblem 3)}
\end{align*}

% \begin{align*}
% \P[h(X)=Y]
% &\ge 
% \P\!\left[h^\star(X)=Y \ \land\ h^\star(X)\in\{h_1(X),\dots,h_T(X)\}\right] - \eps
% &&\text{(Subproblem 3)}\\[2mm]
% &\ge 
% \P\!\left[f^\star(X)=Y \ \land\ f^\star(X)\in\{h_1(X),\dots,h_T(X)\}\right] - 2\eps
% &&\text{(Subproblem 1)}\\[2mm]
% &=
% \P[f^\star(X)=Y]
% -
% \P\!\left[f^\star(X)=Y \ \land\ f^\star(X)\notin\{h_1(X),\dots,h_T(X)\}\right]
% - 2\eps\\[2mm]
% &\ge 
% \P[f^\star(X)=Y] - 3\eps
% &&\text{(Subproblem 2)}\\[2mm]
% &\ge
% \P[h^\star(X)=Y] - 4\eps
% &&\text{(Subproblem 1)}.
% \end{align*}

Thus, the final predictor $h$ is $O(\eps)$-competitive with the optimal hypothesis in $\H$.

\medskip

We next examine each of the three subproblems in turn, describing the obstacles that arise and the ideas that enable us to overcome them.  
While all three components are essential to the overall learner, the label-space reduction (Subproblem 2) is the most technically novel, and we devote particular attention to its construction.

\subsubsection{Approximation by improper finite cover}
\label{sec:reduction-finite-class}

% Algorithmically, this stage proceeds by applying a carefully tailored realizable-case learner to many small sub-samples of the input.  
% Concretely, we run the realizable learner on all sub-datasets of the training sample whose size is at most~$k$, for a parameter~$k$ to be determined later.  
% The underlying intuition is straightforward: for any fixed hypothesis $h^\star \in \mathcal{H}$, there exists a sub-sample consistent with~$h^\star$, and running the realizable learner on such a sub-sample produces a hypothesis~$f^\star$ that agrees with~$h^\star$ on almost all correctly labeled examples.  
% This yields precisely the approximation guarantee required for the finite-class reduction.
%
%Algorithmically, this stage proceeds by applying a carefully tailored realizable-case learner to many small sub-samples of the input.  
%Concretely, we run the realizable learner on all sub-datasets of the training sample whose size is at most~$k$, for a parameter~$k$ to be specified later.  
%The underlying intuition is straightforward: for any fixed hypothesis $h^\star \in \mathcal{H}$, there exists a sub-sample consistent with~$h^\star$, and running the realizable learner on such a sub-sample produces a hypothesis~$f^\star$ that agrees with~$h^\star$ on almost all correctly labeled examples.  

Algorithmically, this stage proceeds by applying a carefully tailored realizable-case learner to many sub-samples of the input. 
The general strategy---constructing a finite proxy class by running a realizable learner on many subsets---has appeared in related contexts (see, e.g.,~\cite{HopkinsKLM24,hanneke:25a}).
The underlying intuition is straightforward: for any fixed hypothesis $h^\star \in \mathcal{H}$, 
%there exists a sub-sample consistent with~$h^\star$, and 
%running the realizable learner on such a sub-sample 
running the realizable learner on the sub-sample %of examples
consistent with $h^\star$ 
produces a hypothesis~$f^\star$ that agrees with~$h^\star$ on almost all examples in the population on which $h^\star$ is correct.
We are most-interested in applying this with $h^\star$
the best-in-class hypothesis.
%Applying this strategy with the best-in-class hypothesis $h^\star \in \mathcal{H}$ would therefore produce a hypothesis $f^\star$ nearly as good as $h^\star$.
However, since we do not know this best-in-class $h^\star$, we do not know which sub-sample to use, so that naively applying this strategy with an arbitrary realizable learner could require applying the learner to all subsets of the data, leading to a potentially-exponential number of hypotheses in the proxy class $\F$. %(corresponding to all subsets of the data).

Fortunately, this problem can be avoided if the realizable learner is a \emph{compression scheme}.
This idea is rooted in the work of \cite{NIPS2016_59f51fd6}.
Indeed, as in that work, 
%Concretely, we run the realizable learner on all sub-datasets of the training sample whose size is at most~$k$, for a parameter~$k$ to be specified later.  
%The underlying intuition is straightforward: for any fixed hypothesis $h^\star \in \mathcal{H}$, there exists a sub-sample consistent with~$h^\star$, and running the realizable learner on such a sub-sample produces a hypothesis~$f^\star$ that agrees with~$h^\star$ on almost all correctly labeled examples.  
%
%This general strategy—constructing a finite proxy class by running a realizable learner on many small subsets—has appeared in related contexts (see, e.g.,~\cite{HopkinsKLM24,hanneke:25a}).  
%A direct implementation of this strategy, however, faces a quantitative obstacle.
%A realizable-case learner requires $k = \Omega(d/\varepsilon)$ examples to guarantee error at most~$\varepsilon$, where $d$ denotes the realizable-dimension of~$\mathcal{H}$.  
%Since the number of sub-samples grows like $n_1^k$, this would lead to classes of size
%\[
%n_1^k \;\ge\; n_1^{\,d/\varepsilon} \;\approx\; n_1^{\,n_1},
%\]
%which is far too large for our purposes.  
%To avoid this exponential dependence on~$n_1$, 
we convert any realizable learner into a sample compression scheme (via boosting), obtaining a compression size
\[
k \;=\; O(d \log n_1).
\]
Under this transformation, 
we need only consider all possible \emph{compression sets} of size $k$, applying the reconstruction function to each such sub-sample, to produce the proxy class.  
This proxy class therefore has size $n_1^k = n_1^{O(d \log n_1)}$, only \emph{quasi-polynomial} in $n_1$.
Moreover, since one of these compression sets corresponds precisely to applying the compression scheme to the sub-sample on which $h^\star$ is correct, 
we can argue that the resulting set $\F$ of functions contains a proxy $f^\star$ with the desired property.

%the number of sub-samples we must consider drops from
%the prohibitive scale $n_1^{n_1}$ to \(n_1^{\,\tilde O(d \log n_1)}\),
%which is only \emph{quasi-polynomial} in~$n_1$.
%This represents an exponential improvement over the naive approach and is crucial for keeping the finite class~$\mathcal{F}$ produced in this stage at a manageable size.

A further central ingredient in obtaining this latter guarantee is the use of a modified loss function in the analysis, which allows us to obtain \emph{realizable-rate} guarantees even though the underlying distribution $P$ need not be realizable.
For any fixed hypothesis $h^\star \in \mathcal{H}$, we consider the loss
\[
l(h,(x,y))
    = \ind\!\left\{\,h^\star(x)=y \;\wedge\; h(x)\neq y\,\right\}.
\]
This loss is used only for analysis (since $h^\star$ is unknown to the learner), but it has a key structural property: with respect to this loss, the proxy $f^\star$ incurs zero loss on the data.
%the class becomes perfectly realizable, because the particular hypothesis $h^\star\in\mathcal{H}$ incurs zero loss.  
This allows us to apply realizable-case guarantees for sample compression schemes when analyzing the generalization error of the proxy $f^\star$ %based on $n_1$ samples 
and is what ultimately yields the required approximation property.

\subsubsection{Learning a list}
\label{sec:label-space-reduction}
In this stage, given a finite hypothesis class $\F$, the goal is to construct a list $\{ h_1,\ldots,h_{n_2} \} \subseteq \F$ such that for every $f^\star \in \F$, with high probability,
\begin{align*}
    \P\bigl(f^\star(X)=Y \ \land\ f^\star(X)\notin\{h_1(X),\dots,h_{n_2}(X)\}\bigr) \le \eps.
\end{align*}
The significant challenge arises when we require achieving this guarantee using only the realizable sample complexity for $\F$; namely
\begin{align*}
    n_2 \approx \frac{\log |\F|}{\eps}.    
\end{align*}
We accomplish this by employing an \emph{online learning} mechanism for $T=n_2$ rounds where the list $\{ h_1,\ldots,h_T \}$ is constructed gradually; namely, in each round $t \in [T]$ a new example $(x_t,y_t)$ is processed and a new hypothesis is added to the list. More specifically, we employ the \emph{Multiplicative Weights} (MW) algorithm \citep{littlestone:94,freund:97b} which is widely used in the online learning literature, over reward functions designed to adaptively approximate our objective of interest (in expectation over $(x_t,y_t) \sim P$). The MW algorithm can be summarized as follows:
% \begin{align*}
%     r_t(h) = \ind{\{h(x_t) = y_t \land h_s(x_t) \neq y_t\  \forall s\in[t-1]\}}.
% \end{align*} 

\vspace{\baselineskip}
\begin{algorithm}[H]
\caption{Multiplicative Weights (MW)} \label{alg:mul_weights_overview}
Initialize $p_1$ as the uniform distribution over $\F$, step size $\eta>0$\;
\For{$t\in[T]$}{
    Sample $h_t\sim p_t$\; %independently\; %%% SH: it's a conditional distribution given previous h_s's, so "independently" doesn't make sense
    Let $r_t(h) = \ind{\{h(x_t) = y_t \land h_s(x_t) \neq y_t\  \forall s\in[t-1]\}} \quad \forall h \in \F$\;
    % \textcolor{red}{More direct to use rewards:
    % \begin{align*}
    %     r_t(h;x_t,y_t) = \ind{\{h(x_t) = y_t \land h_s(x_t) \neq y_t \quad \forall s < t\}}
    % \end{align*}
    % }
    Update $p_{t+1}(h) \; \propto \; p_t(h) e^{\eta r_t(h)} \quad \forall h \in \F$\;
}
\end{algorithm}
\vspace{\baselineskip}

Note that in this approach, $h_t$ is a function of only the first $t-1$ examples $(x_1,y_1),\ldots,(x_{t-1},y_{t-1})$. While processing the data in such an online manner may seem limiting, it is in fact unclear how we could achieve similar results using the entire batch of $T$ data points at once.
For instance, via VC dimension, one can show uniform (multiplicative Chernoff) concentration of losses $\ind{\{ h(x)=y \land h(x) \notin \{h_1(x),\ldots,h_T(x)\} \}}$ over $h,h_1,\ldots,h_T \in \F$
can sometimes require sample complexity at least $T/\eps \approx \log(|\F|)/\eps^2$, too large for our purposes. 
%%% The example witnessing this puts O(eps) mass on O(T) points, and each h_i is correct on a unique one of them, and h is correct on all of them, so if we don't have enough samples to get a copy of most of the O(T) points, the list that omits those missing ones has 0 empirical loss but > eps population loss.  (though it doesn't say whether smaller T might suffice?  Can it be as bad as Tlog|F|/eps?  I know the list size T can't be smaller than 1/eps)
Furthermore, using an online algorithm to achieve a purely stochastic objective is a well-known methodology in the stochastic optimization literature, often referred to as \emph{online-to-batch conversion} (see, e.g., \cite{cesa2004generalization,cesa-bianchi:06}). 

% While it may initially seem like a natural attempt to solve \cref{subproblem:label-space-reduction} by defining an appropriate empirical objective which depends on the entire batch of $n_2$ examples, it is very unclear how such an objective can be suitably defined and efficiently optimized. Perhaps surprisingly, we achieve this step by employing an \emph{online learning} mechanism for $T=n_2$ rounds which processes a single example at a time; such that each list member $h_t$ depends only on $t-1$ examples $(x_1,y_1),\ldots,(x_{t-1},y_{t-1})$. 

The reward functions $r_t(\cdot)$ are defined such that for every $f^\star \in \F$,
\begin{align*}
    \E_{(X,Y) \sim P}[r_{T}(f^\star)] = \P\bigl(f^\star(X)=Y \ \land\ f^\star(X)\notin\{h_1(X),\dots,h_{T-1}(X)\}\bigr), 
\end{align*}
so it is evident that the goal in this stage is essentially to upper bound the expected final reward of~$f^\star$. We achieve this goal by analyzing the \emph{regret} of the MW algorithm, but we emphasize that for our purposes the regret analysis functions only as a means towards achieving our goal, and is not important in and of itself. As we will soon argue, our rewards are structured in such a way that every online algorithm achieves trivial cumulative reward in expectation, and the MW regret bound is used in tandem with this fact to deduce that every $f^\star \in \F$ also has small expected reward.

In a bit more detail, a key property of the rewards, which turns out to be crucial in our analysis, is that they depend on the past decisions of the online algorithm via $h_1,\ldots,h_{t-1}$, so that even though $(x_t,y_t)$ is sampled i.i.d.\ from $P$, the rewards are inherently \emph{adaptive}. To see why this type of adaptivity actually works in our favor, we note that the expected cumulative reward of the algorithm is bounded; namely
\begin{align*}
    \E \left[ \textstyle\sum_{t=1}^T r_t(h_t) \right] \leq 1.
\end{align*}
Indeed, a standard coupling argument shows that
\begin{align*}
    \E \left[ \textstyle\sum_{t=1}^T r_t(h_t) \right] &= \E \left[ \textstyle\sum_{t=1}^T \ind{\{h_t(x_t) = y_t \land h_s(x_t) \neq y_t \quad \forall s\in[t-1]\}} \right] \\
    &=
    \E_{(X,Y) \sim P} \left[ \textstyle\sum_{t=1}^T \ind{\{h_t(X) = Y \land h_s(X) \neq Y \quad \forall s\in[t-1]\}} \right] \\
    &=
    \P_{(X,Y) \sim P} \left( \exists t \in [T] \text{ s.t. } h_t(X)=Y \right).
\end{align*}
In fact, this property holds \emph{no matter which online algorithm is used}, as it is a direct consequence of the structure of the rewards and how they adapt to the algorithm's past decisions. 

This observation already reveals a stark contrast between our analysis and other typical uses of online learning algorithms in the literature: while standard arguments often attempt to provide good guarantees on the cumulative reward of the online algorithm, in our case we use regret guarantees to bound the cumulative reward of the best predictor (classifier) in hindsight.
Indeed, the rewards~$r_t$ themselves are constructed in such a way that \emph{any} online algorithm inherently gains almost no reward in expectation;
% The benefit from this observation is that, 
when the online algorithm being employed is a no-regret algorithm with a suitable regret bound (such as MW), we can infer that the cumulative reward of \emph{any fixed classifier~$f^\star \in \F$} satisfies
% The benefit from this observation is that the specific choice of MW (with an appropriate choice of step size) yields a regret bound asserting that the cumulative reward of MW is not much smaller than that of any fixed classifier $f^\star \in \F$; namely
\begin{align*}
    \E \left[ \textstyle\sum_{t=1}^T r_t(f^\star) \right] \leq 2 \E \left[\textstyle\sum_{t=1}^T r_t(h_t) \right] + 4 \log |\F| \leq 2 + 4 \log |\F|,
\end{align*}
implying that $f^\star$ also has small cumulative reward in expectation. 
We remark here that it is necessary for us to make use of a \emph{multiplicative}-type regret bound of MW (comparing the cumulative reward of MW to a constant times the reward of $f^\star$), rather than more standard additive regret bounds that result in an additive penalty of $\approx \sqrt{T \log |\F|}$, which we cannot allow, as it would result in sample complexity of
% \begin{align*}
$
    \smash{\widetilde O} ( \log(|\F|)/\eps^2 ) = \smash{\widetilde O} ( d/\eps^2 ).
$
% \end{align*}

The next step in our analysis revolves around converting the guarantee on the expected cumulative reward of $f^\star$ to a guarantee on its final reward $r_T(f^\star)$. This is where an additional structural property of the rewards is helpful; namely, they are pointwise monotonically non-increasing in expectation:
\begin{align*}
    \E \left[ r_{t+1}(f)\right] \le \E \left[ r_t(f) \right] ~, 
    \quad \forall f \in \F.
\end{align*}
This monotonicity allows us to argue directly about the final reward of $f^\star$ rather than the time-averaged reward as is more commonly the case with online-to-batch reductions. Indeed, putting all of the above together implies that
\begin{align*}
    \E \left[ r_T(f^\star) \right] \leq \frac{2 + 4 \log |\F|}{T},
\end{align*}
which is at most $\eps$ for $T = \Theta \left( \frac{\log |\F|}{\eps} \right)$.

\subsubsection{Learning a list-bounded classifier}
\label{sec:classifier-from-list}

In this stage, we exploit the list produced by the label-space reduction step.
A natural (but naive) idea is to restrict the hypothesis class~$\mathcal{H}$ to those predictors~$h$ that always output a label from the list:
\[
h(x) \in \text{list}(x)
\qquad\text{for all } x.
\]
If this were possible, then the effective label-space size would be at most~$T$, while the Natarajan dimension would not increase (since we are passing to a subclass of~$\mathcal{H}$).  
Classical multiclass PAC bounds would then yield a learning rate of
\[
O\!\left(\frac{\mathtt{Nat}\cdot \log T}{\varepsilon^2}\right),
\]
which is precisely the behavior we need in the third stage of the algorithm\footnote{While plugging in the value of $T$ gives a dependence of $(\mathtt{Nat}\cdot \log (d/\eps)) / \eps^2$, the logarithmic factor in $d/\eps$ can replaced by a factor of $\log(\mathtt{Nat}/\eps)$ when $\eps$ is small, and if $\eps$ is large the additive $d / \eps$ term dominates the final bound.}.

However, this naive restriction is typically \emph{empty}: in general there may be no hypothesis~$h \in \mathcal{H}$ that always predicts within the list.
To circumvent this difficulty, we instead work with \emph{partial functions}.
For each $h \in \mathcal{H}$, we define a partial predictor $\tilde h$ that agrees with~$h$ on those points~$x$ for which $h(x) \in \text{list}(x)$, and is undefined elsewhere.
By the guarantee provided by the label-space reduction stage, the optimal hypothesis $h^\star$ is mapped to a partial function~$\tilde h^\star$ that is defined on a large fraction of the domain—namely, on a set of measure at least
\[
1 - \operatorname{er}_P(h^\star) - O(\varepsilon).
\]

The resulting collection of partial predictors is no longer a standard hypothesis class, and classical Natarajan-dimension arguments do not apply directly.
Instead, we appeal to the theory of \emph{partial concept classes}, developed by \cite{long01,partial_concept}, which provides generalization guarantees based on one-inclusion graph predictors and sample compression techniques.
Applying this machinery yields the desired agnostic learning rate while restricting predictions to the labels appearing in the learned list. We remark that for partial concept classes it is known that proper learning is not possible even in the binary case \citep{partial_concept}, which aligns with the previous results of \cite{daniely2014optimal} who show that multiclass classification over unbounded label spaces is not properly learnable in general.

\subsection{Connections to other recent results}

Theorem~\ref{thm:main-informal} may be viewed as an instance of a broader phenomenon observed in certain recent works from multiple branches of the learning theory literature, 
which find distinctions between the primary dimensions characterizing the realizable and agnostic sample complexities.
Specifically, \cite{erez:24} study the problem of multiclass PAC learning with \emph{bandit feedback}, 
where the optimal realizable-case sample complexity 
is lower-bounded by $\frac{K+\d}{\eps}$
where $\d = \d_{\n}(\H)$ and
$K = \max_x |\{ h(x) : h \in \H\}|$ is the 
effective label-space size \cite{daniely:15}.
In contrast, \cite{erez:24} find that, in the agnostic case, 
the asymptotically-relevant (small $\eps$) term is 
of the form $\tilde{O}\!\left( \frac{\d}{\eps^2} \right)$, which is most notable for the absence of a 
linear dependence on $K$ (unlike the realizable case).
Similarly, the recent work of \cite{hanneke:25b} 
studies the problem of \emph{active learning} (with $K < \infty$),
where the realizable-case sample complexity
is known to offer improvements over passive learning if and only if a dimension of $\H$ called the \emph{star number} is bounded \cite{hanneke:15b}.
However, for the agnostic case, \cite{hanneke:25b} 
proves a bound 
$\tilde{O}\!\left( \d \frac{\er_P(h^\star)^2}{\eps^2} + \frac{\d}{\eps} \right)$,
which in the non-realizable regime of interest ($\eps \ll \er_P(h^\star) \ll 1$) reflects improvements over the
sample complexity of passive learning $\tilde{\Theta}\!\left( \d \frac{\er_P(h^\star)}{\eps^2} \right)$
for \emph{every} concept class, \emph{regardless} of the star number.
%A similar pattern is also reflected in earlier work 
%on active learning by \cite{hanneke:15b} in the case of special noise models (Tsybakov noise and Benign noise).

A related phenomenon was also observed in the work of \cite{carmon:24}, which studies the problem of stochastic optimization with bounded convex Lipschitz functions with norm-bounded parameter vectors in $\mathbb{R}^d$ (among other results).  While \cite{feldman:16} has shown the sample complexity of uniform convergence to be $\tilde{\Theta}\!\left(\frac{d}{\eps^2}\right)$,
\cite{carmon:24} found that ERM achieves 
sample complexity $\tilde{\Theta}\!\left( \frac{d}{\eps} + \frac{1}{\eps^2} \right)$, so that the asymptotically-relevant term (small $\eps$) has no dependence on dimension $d$.

In the other direction, the work of \cite{attias:23} 
finds a separation between realizable and agnostic learning of $\mathbb{R}$-valued functions with $\ell_p$ losses.
However, the results of that work reflect perhaps the more-expected relation, 
where the dimension characterizing the realizable 
sample complexity can be arbitrarily smaller than the dimension characterizing the agnostic case.

\subsection{Organization}
%\textcolor{red}{TODO: explain how the rest of the paper is organized}
\cref{sec:definitions-of-sample-complexity} presents the multiclass PAC learning framework. 
The Natarajan dimension, DS dimension, and realizable dimension are defined in \cref{sec:dim}. \cref{sec:sample-compression} introduces sample compression schemes which will be used in our proposed algorithm. 
We state the formal version of our main theorem and its proof in \cref{sec:main}. \cref{sec:main1,sec:main2,sec:main3} formalize the three phases of our proposed algorithm. \cref{sec:main4} summarizes the algorithm and completes the proof. 
\cref{sec:example} illustrates the necessity of local guarantee for list learning with an example. \cref{sec:discuss_dim} further discusses our definition of the realizable dimension. 
\cref{sec:compression,sec:MW_lemma,sec:partial_sample_compression} supplement the technical details in \cref{sec:main1,sec:main2,sec:main3}. 
\cref{sec:lb} proves our proposed lower bound (which uses the realizable dimension) of the agnostic sample complexity. 
\cref{sec:lemmas} contains the technical lemmas required by previous sections.

\section{Preliminaries} \label{sec:prelim}

\paragraph{Notation}
%Define the function $\log:\R\to[1,\infty)$, $t\mapsto\log_e(\max\{t,e\})$. 
For any $n\in\N$, let $[n]$ denote the set $\{1,\dots,n\}$. 
Let $\X$ and $\Y$ denote the input and label spaces respectively. 
Define $\Z:=\X\times\Y$. 
A classifier is a function from $\X$ to $\Y$. 
A hypothesis class $\H$ is a subset of $\Y^\X$. For any $\bx=(x_1,\dots,x_n)\in\X^n$, the projection of $\H$ to $\bx$ is $\H|_{\bx}:=\{(h(x_1),\dots,h(x_n)):h\in\H\}\subseteq\Y^n$. 
For any sequence or set $S$, let $S^*:=\cup_{k=0}^\infty S^k$ denote the set of all sequences whose elements belong to $S$. 

\subsection{Multiclass PAC learning}
\label{sec:definitions-of-sample-complexity}
For any distribution $P$ over $\Z=\X\times\Y$, the error rate of a classifier $h\in\Y^\X$ under $P$ is defined as 
\begin{align*}
\er_{P}(h):=P(\{(x,y)\in\X\times\Y:y\neq h(x)\}).
\end{align*}
Given a sequence $\bs\in\Z^n$ with $n\in\N$, the empirical error rate of $h$ on $\bs$ is defined as 
\begin{align*}
\er_{\bs}(h):=\tfrac{1}{n}\textstyle\sum_{(x,y)\in\bs}\ind\{y\neq h(x)\}.
\end{align*}
We define multiclass learners and PAC learning as follows.
\begin{definition}[Multiclass learner]
A \textbf{multiclass learner} (or a learner) $\A$ is a (possibly randomized) algorithm which given a sequence $\bs\in\Z^*=\cup_{n=0}^\infty(\X\times\Y)^n$ and a hypothesis class $\cH\subseteq\Y^\X$, outputs a classifier $\A(\bs,\cH)\in\Y^\X$. 
\end{definition}
\begin{definition}[Multiclass Agnostic PAC learning] 
For any hypothesis class $\cH\subseteq\Y^\X$, the \textbf{agnostic (PAC) sample complexity} $\M_{\A,\cH}^{\AG}:(0,1)^2\to\N$ of a multiclass learner $\A$ is a mapping from $(\epsilon,\delta)\in(0,1)^2$ to the smallest positive integer such that for any $m\ge\M_{\A,\cH}^{\AG}(\epsilon,\delta)$ and any distribution $P$ over $\Z=\X\times\Y$, 
$$
\P_{S\sim P^{m},\A}(\er_P(\A(S,\cH))-\textstyle\inf_{h\in\H}\er_P(h)>\epsilon)\le\delta,
$$ 
and we define $\M_{\A,\cH}^{\AG}(\epsilon,\delta)=\infty$ if no such integer exists. 
We say $\cH$ is \textbf{agnostic PAC learnable} by $\A$ if $\M_{\A,\cH}^{\AG}(\epsilon,\delta)<\infty$ for all $(\epsilon,\delta)\in(0,1)^2$. 
The \textbf{agnostic (PAC) sample complexity} of $\cH$ is defined as $\M_{\cH}^{\AG}(\epsilon,\delta):=\inf_{\A}\M_{\A,\cH}^{\AG}(\epsilon,\delta)$ for any $(\epsilon,\delta)\in(0,1)^2$. 
\end{definition}
In the above definition, $P$ can be any distribution over $\Z$. 
When restricting $P$ to be realizable as defined below, we obtain the definition of multiclass realizable PAC learning. 
\begin{definition}[Realizability] \label{def:real}
We say a sequence $\bs\in\Z^*$ is \textbf{realizable by} $h\in\Y^\X$ (or $h$ \textbf{realizes} $\bs$) if $y=h(x)$ for all $(x,y)\in\bs$. 
Define $\bs(h)$ to be the longest subsequence of $\bs$ realizable by $h$. 
We say $\bs$ is $\cH$\textbf{-realizable} if $\sup_{h\in\H}|\bs(h)|=n$. 
We say a distribution $P$ over $\Z$ is $\cH$\textbf{-realizable} if $\inf_{h\in\cH}\er_P(h)=0$. 
Define $\RE(\H)$ to be the set of all $\H$-realizable distributions. 
\end{definition}
\begin{definition}[Multiclass Realizable PAC learning] 
For any hypothesis class $\cH\subseteq\Y^\X$, the \textbf{realizable (PAC) sample complexity} $\M_{\A,\cH}^\RE:(0,1)^2\to\N$ of a multiclass learner $\A$ is a mapping from $(\epsilon,\delta)\in(0,1)^2$ to the smallest positive integer such that for any $m\ge\M_{\A,\cH}^\RE(\epsilon,\delta)$ and any distribution $P\in\RE(\cH)$, $\P_{S\sim P^{m},\A}(\er_P(\A(S,\cH))>\epsilon)\le\delta$, and we define $\M_{\A,\cH}^\RE(\epsilon,\delta)=\infty$ if no such integer exists. 
We say $\cH$ is \textbf{realizable PAC learnable} by $\A$ if $\M_{\A,\cH}^\RE(\epsilon,\delta)<\infty$ for all $(\epsilon,\delta)\in(0,1)^2$. 
The \textbf{realizable (PAC) sample complexity} of $\cH$ is defined as $\M_{\cH}^\RE(\epsilon,\delta):=\inf_{\A}\M_{\A,\cH}^\RE(\epsilon,\delta)$ for any $(\epsilon,\delta)\in(0,1)^2$. 
\end{definition}
For future reference, it is convenient to define the expected error rate of an learner $\A$ under a distribution $P$ over $\Z$,
\begin{align*}
\epsilon_{\A,\cH,P}:\N\to[0,1],\ \  n\mapsto\E_{S\sim P^n,\A}[\er_{P}(\A(S,\cH))]=\P_{(S,(X,Y))\sim P^{n+1},\A}(Y\neq\A(S,\cH)(X))
\end{align*}
We will need $\epsilon_{\A,\cH}^\RE:=\sup_{P\in\RE(\cH)}\epsilon_{\A,\cH,P}$ and $\epsilon_{\cH}^\RE:=\inf_{\A}\epsilon_{\A,\cH}^\RE$ for the next section.  

\subsection{Dimensions of hypothesis classes} \label{sec:dim}
In this section, we formally define some combinatorial dimensions of hypothesis classes involved in our sample complexity bounds. 
The first one is the Natarajan dimension which was proposed as an extension of the VC dimension \cite{Vapnik68} in the sense of requiring the hypothesis class to contain a copy of a ``Boolean cube'' ($\{0,1\}^d$ for $d\in\N$). 
\begin{definition}[Natarajan dimension, \cite{natarajan:89}]
We say that $\bx\in\X^d$ is \textbf{N-shattered} by $\cH\subseteq\Y^\X$ if there exists $f,g:[d]\to\Y$ such that for all $i\in[d]$, we have $f(i)\neq g(i)$ and $\H|_{\bx}\supseteq\{f(1),g(1)\}\times\cdots\times\{f(d),g(d)\}$.
The \textbf{Natarajan dimension} $\d_{\n}(\H)$ of $\H$ is the maximum size of an N-shattered sequence.
\end{definition}
The DS dimension is another extension of the VC dimension, where a ``pseudo cube'' is considered instead of a Boolean cube. 
\begin{definition}[Pseudo-cube, \cite{brukhim2022characterization}]
For any $d\in\N$, a set $\H\subseteq\Y^d$ is called a \textbf{pseudo-cube} of dimension $d$ if it is non-empty, finite, and for every $h\in\H$ and $i\in[d]$, there exist at least an $i$-neighbor of $h$ in $\H$, where $g$ is an $i$-neighbor of $h$ if $g(i)\neq h(i)$ and $g(j)=h(j)$ for all $j\in[d]\backslash\{i\}$. 
\end{definition}
\begin{definition}[DS dimension, \cite{daniely2014optimal}] \label{def:DS}
For any $d\in\N$, we say $\bx\in\X^d$ is \textbf{DS shattered} by $\cH\subseteq\Y^\X$ if $\H|_{\bx}$ contains a $d$-dimensional pseudo-cube. The \textbf{DS dimension} $\d_{\DS}(\H)$ of $\H$ is the maximum size of a DS shattered sequence. 
\end{definition}
Though DS dimension has been shown to characterize multiclass PAC learnability \cite{brukhim2022characterization}, there still exists a polynomial gap in terms of the DS dimension between the upper and lower bounds of PAC sample complexity even in the realizable setting \cite{hanneke2024improved}. Since we work on developing algorithms for the agnostic setting with access to realizable learners, we would like to abstract the optimal realizable sample complexity with constant error threshold as a dimension to avoid the current gap in realizable learning. 
Then, our agnostic sample complexity will depend on this dimension, and any improvement on the realizable sample complexity will directly improve our agnostic sample complexity via this dimension. In this sense, we call this dimension the ``realizable dimension'', which can be controlled by the DS dimension as existing results have shown. 
\begin{definition} \label{def:realizable_dim}
For any hypothesis class $\H\subseteq\Y^\X$, learner $\A$, and constant $r\in(0,1/2)$, we define 
\begin{align*}
\d_{\RE}(\A,\H,r):=\inf\{n\in\N:\eps_{\A,\H}^\RE(n)\le r\}
\end{align*}
which captures the minimum sample size required by $\A$ to learn $\H$ with constant error $r$. 
Taking the infimum over all learners, we obtain $\d_{\RE}(\H,r):=\inf_{\A}\d_{\RE}(\A,\H,r)$. 
% Taking the infimum over all deterministic learners, we obtain $\d_{\RE}^{\det}(\H,r):=\inf_{\A\textup{ is deterministic}}\d_\RE(\A,\H,r)$. 
The \textbf{(multiclass) realizable dimension} of $\H$ is defined as $\d_{\RE}(\H):=\d_{\RE}(\H,1/9e)$. 
\end{definition}
%
%We now discuss on the relationship between the DS dimension and the realizable functions. 
For any constant $r\in(0,1/2)$, existing results (see e.g., \cite[Theorem 1.9]{hanneke2024improved}) show that \begin{align} \label{eq:real_dim_bound0}
\Omega(d_{DS})\le \d_{\RE}(\H,r)\le O((d_{DS})^{3/2}\log^2(d_{DS})),\ \textup{where }d_{DS}:=\d_{\DS}(\H).
\end{align}
We briefly comment on the choice of $1/9e$ in the above definition of $\d_\RE(\H)$. 
As is observed from the definition, $\d_{\RE}(\H,r)$ characterizes the sample complexity of weak realizable learners with constant error $r$. 
Later, we will apply such $r$-weak learners to construct sample compression schemes (see \cref{def:compression}), which is a key building block of our learning algorithm, via the boosting technique in \cite{NIPS2016_59f51fd6}. 
Note that this technique works for any constant $r\in(0,1/2)$ for constructing the sample compression scheme. 
However, for technical reasons, it is more convenient to boost deterministic learners. Thus, we are actually interested in $\d^{\det}_\RE(\H,r):=\inf_{\A\textup{ deterministic }}\d_\RE(\A,\H,r)$. 
%Different from $\d_{\RE}(\cdot,r)$, we restrict to deterministic learners in the definition of the realizable dimension $\d_{\RE}(\cdot)$. 
%This is for technical convenience: deterministic learners can be directly applied to construct sample compression schemes (see \cref{def:compression}) via the boosting technique in \cite{NIPS2016_59f51fd6}, which is a key building block of our learning algorithm. 
Though $\d_{\RE}^{\det}(\H,r)\ge\d_{\RE}(\H,r)$, we can still bound $\d_{\RE}^{\det}(\H,r)$ by $\d_{\RE}(\H,r')$ up to constant factors for smaller constant $r'\in(0,r)$ through existing results \cite{10.1145/3564246.3585190,10353103}:  
\begin{align} \label{eq:real_dim_bound}
\Omega(d_{DS})\le \d_{\RE}^{\det}(\H,r)\le O(\d_{\RE}(\H,r/3e))\le O((d_{DS})^{3/2}\log^2(d_{DS})).
\end{align}
% It is worth mentioning that the choice of $1/3$ in the definition of $\d_{\RE}(\H)$ is not important. Any constant $r\in(0,1/2)$ will work for our purpose and if that $r$ is used to define $\d_{\RE}(\H)$, \eqref{eq:real_dim_bound} still holds with $1/12e$ replaced by $r/4e$.
$1/9e$ is obtained by setting $r=1/3$. 
Then, there exist deterministic learners whose sample complexity for constant error $1/3$ is bounded by $\d_{\RE}(\H)$. But can we construct such learners algorithmically for our algorithm to work? The answer is yes: the upper bound in \eqref{eq:real_dim_bound} is proved via a specific deterministic learner known in the literature as the ``one-inclusion algorithm'' \cite{NIPS2006_a11ce019,daniely2014optimal,brukhim2022characterization} $\A_{\oig}$ (see \cref{alg:oia}), which directly makes our learning algorithm explicit. 
We refer readers to \cref{sec:discuss_dim} for more details of the above discussion and a proof of \eqref{eq:real_dim_bound} using existing results.

\subsection{Sample compression scheme} \label{sec:sample-compression}
Our learning algorithm relies on the idea of sample compression schemes \cite{floyd:95,NIPS2016_59f51fd6} where we first select a ``compressed'' subset of the input samples and then learn a classifier that is consistent with the input samples using only this subset.
The generalization error of sample compression schemes can be controlled by the size of the compressed subset: smaller size corresponds to smaller generalization error. 
In this paper, we will construct sample compression schemes under different loss functions and apply them as building blocks of our learning algorithm. 
We first introduce the concept of selection schemes which depicts the basic structure of sample compression schemes. 
\begin{definition}[Selection scheme] \label{def:selection}
A \textbf{selection scheme} $\Ascr=(\kappa,\rho)$ is a pair of maps $(\kappa,\rho)$ where
\begin{itemize}
    \item $\kappa$ is called a selection map which given an input $\bs\in\Z^*$, selects a sequence consisting of elements of $\bs$;
    \item $\rho:\Z^*\to\Y^\X$ is called a reconstruction map. 
\end{itemize}
For an input sequence $\bs\in\Z^*$, the selection scheme $\Ascr$ outputs a classifier $\Ascr(\bs)\in\Y^\X$. 
The \textbf{size} of $\Ascr$ is defined as $\size_{\Ascr}(n):=\sup_{\bs\in\cup_{m\in[n]}\Z^m}|\kappa(\bs)|$. 
If $\Ascr$ makes use of a hypothesis class $\H\subseteq\Y^\X$, we may write its output as $\Ascr(\bs,\H)$ for clarity. Similarly, $\kappa(\bs,\H)$ and $\rho(\bs,\H)$ are also used. 
\end{definition}
Sample compression schemes are not restricted to the setting of zero-one loss (classification) and can be defined generally as selection schemes that interpolates the input samples under the given loss function. Formally, let $l:\Y^{\X}\times\Z\to[0,1]$ be a loss function. For $h\in\Y^\X$, $\bs\in\Z^n$ with $n\in\N$, and distribution $P$ over $\Z$, we define the empirical and expected risks of $h$ under $l$ as 
\begin{align*}
L_{\bs}(h):=\tfrac{1}{n}\textstyle\sum_{(x,y)\in\bs}l(h,(x,y))
\quad\textup{and}\quad
L_P(h):=\E_{(X,Y)\sim P}[l(h,(X,Y))].
\end{align*}
\begin{definition}[Sample compression scheme] \label{def:compression}
Given a loss function $l$, a selection scheme $\Ascr$ is called an $l$\textbf{-sample compression scheme for} $\H\subseteq\Y^\X$ if 
\begin{align*}
L_{\bs}(\Ascr(\bs))\le \textstyle\inf_{h\in \H}L_{\bs}(h),\ \forall \bs\in\Z^n,\ n\in\N.
\end{align*}
If $l$ is the zero-one loss $(h,(x,y))\mapsto\ind\{h(x)\neq y\}$, $\forall h\in\Y^\X$ and $(x,y)\in\Z$, we call $\Ascr$ a \textbf{sample compression scheme for }$\H$. 
\end{definition}
It is worth pointing out that a selection scheme $\Ascr=(\kappa,\rho)$ such that $L_{\bs}(\Ascr(\bs))=0$ for all $\H$-realizable sequences $\bs\in\Z^*$ directly induces a sample compression scheme for $\H$, defined as follows. For any $\bs\in\Z^*$, let $h_{\bs}$ minimizes $\er_{\bs}(h)$ among all $h\in\H$. Then, the selection map $\kappa'(\bs):=\kappa(\bs(h_{\bs}))$ leads to a sample compression scheme $(\kappa',\rho)$ for $\H$. 
This conclusion also holds for other losses taking value in $\{0,1\}$. Thus, it suffices to focus on ``realizable'' sequences in such cases.

\section{Main results} \label{sec:main}
The main contribution of this paper is given by the following theorem. 
\begin{theorem} \label{thm:main}
Let $\H\subseteq\Y^\X$ have Natarajan dimension $\d_\n(\H)=d_N\in\N$ and DS dimension $\d_{\DS}(\H)=d_{DS}\in\N$. Then, $d:=\d_{\RE}(\H)$ satisfies that $\Omega(d_{DS})\le d\le O((d_{DS})^{3/2}\log^2(d_{DS}))$. 
We have the following lower and upper bounds on the agnostic PAC sample complexity of $\H$,
\begin{align}
\label{eq:main}
\Omega\left(\frac{d_N+\log(1/\delta)}{\eps^2}+\frac{d}{\eps}\right)
\le \M_{\H}^{\AG}(\eps,\delta)
\le 
O\bigg(&\frac{d_N\log^3(d_N/\eps)+\log(1/\delta)}{\eps^2}
+\frac{d\log^2(d/\eps)}{\eps}\bigg).
\end{align}
%If $d\le C(d_{DS})^{1+\alpha}\log^{\beta}(d_{DS})$ for constants $\alpha>0$, $\beta\ge0$, and $C\ge 1$, then, the upper bound in \eqref{eq:main} can be simplified as follows,
%\begin{align} \label{eq:current_ub}
%\M_{\H}^{\AG}(\eps,\delta)\le O\bigg(
%\frac{d_N\log^3(1/\eps)+\log(1/\delta)}{\eps^2}
%+\frac{(d_{DS})^{1+\alpha}\log^\beta(d_{DS})\log^2(d_{DS}/\eps)}{\eps}
%\bigg)
%\end{align}
\eqref{eq:main} immediately implies that 
\begin{align}
\label{eq:current_ub}
\M_{\H}^{\AG}(\eps,\delta)\le O\bigg(
\frac{d_N\log^3(1/\eps)+\log(1/\delta)}{\eps^2}
+\frac{(d_{DS})^{3/2}\log^2(d_{DS})\log^2(d_{DS}/\eps)}{\eps}
\bigg).
\end{align}
Furthermore, if one establishes $d=\tilde O(d_{DS})$, it follows that $\M_{\H}^{\AG}(\eps,\delta)=\tilde\Theta\left(\frac{d_N+\log(1/\delta)}{\eps^2}+\frac{d_{DS}}{\eps}\right)$. 
\end{theorem}
We prove the second inequality in \eqref{eq:main} by constructing a multiclass learner (see \cref{alg:main}) whose agnostic sample complexity is bounded by the upper bound listed in \eqref{eq:main}. 
In the following texts, we first establish the building blocks for our proposed agnostic learner in \cref{sec:main1,sec:main2,sec:main3}, which consist the three phases of the proposed learner. 
Then in \cref{sec:main4}, we present the pseudo code of the proposed learner and prove the upper bounds in \cref{thm:main}. 

For the lower bound in \eqref{eq:main}, the $\Omega((d_N+\log(1/\delta))/\eps^2)$ term has been established in the literature, c.f. \cite[Theorem 3.5]{daniely:15}. The $\Omega(d/\eps)$ term can be established following the idea in \cite[Section A]{hanneke2024improved} for constructing hard distributions. We defer the proof of the lower bound to \cref{sec:lb}. 

\subsection{Constructing a finite sets of classifiers} \label{sec:main1}
The initial step of our learner constructs a finite set of classifiers with \cref{alg:CC} to ``approximate'' the original possibly infinite hypothesis class in terms of the correct regions of the classifiers therein. 
This is achieved via sample compression schemes. \cref{thm:CC} provides the theoretical guarantees of the finite set constructed in \cref{alg:CC}. 

\begin{algorithm}[H]
\caption{Compression-based classifiers $\mathsf{CC}(\bs,\H,\Ascr)$} \label{alg:CC}
\SetKwInput{KwInput}{Input}
\SetKwInput{KwOutput}{Output}
\SetKwComment{Comment}{/* }{ */}
\KwInput{Sequence of examples $\bs\in\Z^*$, hypothesis class $\H\subseteq\Y^{\X}$, selection scheme $\Ascr=(\kappa,\rho)$.}
\KwOutput{A finite set $\F$ of classifiers.}
Initialize $\F\gets\{\}$\;
\For{$\bt\in\bigcup_{0\le k\le \size_{\Ascr}(|\bs|)}\bs^k$}{
    Add $\rho(\bt,\H)$ into $\F$;
}
\end{algorithm}
\begin{theorem} \label{thm:CC}
Let $\Ascr(\cdot,\H)$ be a deterministic sample compression scheme as defined in \cref{def:compression} of size $\size_\Ascr(n)$. 
%Without loss of generality, we can assume that the size of the output sequence of selection map of $\Ascr$ is $\size_\Ascr(n)$ for any input $\bs\in\cup_{m\in[n]}\Z^m$. 
Then, the output set $\mathsf{CC}(\bs,\H,\Ascr)$ of \cref{alg:CC} on $\bs\in\Z^n$ has size at most $n^{\size_\Ascr(n)+1}$. 
Moreover, $f_{h,\bs,\Ascr}:=\Ascr(\bs(h),\H)\in \mathsf{CC}(\bs,\H,\Ascr)$ for all $h\in\Y^\X$. 

For any distribution $P$ over $\Z$, sample $S\sim P^n$ with $n$. Then, there exists some universal constant $C_1>0$ such that for any $h\in\H$, with probability at least $1-\delta$ for $\delta\in(0,1)$, we have
\begin{align} \label{eq:classifier_compression}
P(\{(x,y)\in\Z:h(x)=y\neq f_{h,S,\Ascr}(x)\})\le C_1\frac{\size_\Ascr(n)\log(n)+\log(1/\delta)}{n}.
\end{align}
\end{theorem}
\begin{proof}
Let $\kappa$ and $\rho$ denotes the selection and reconstruction maps of $\Ascr$ respectively. 
Since $\Ascr$ is deterministic, $\rho$ is also deterministic and $\mathsf{CC}(\bs,\H,\Ascr)\subseteq\cup_{0\le i\le \size_\Ascr(n)}\{\rho(\bt,\H):\bt\in\bs^i\}$. It follows that
\begin{align*}
|\mathsf{CC}(\bs,\H,\Ascr)|\le n^{0}+n^1+\dots+n^{\size_{\Ascr}(n)}\le n^{\size_{\Ascr}(n)+1}.
\end{align*}
By the definition of size in \cref{def:selection}, we have $\kappa(\bs(h))\in\cup_{0\le k\le \size_{\Ascr}(n)}\bs^k$ for any $h\in\Y^\X$. Thus, $f_{h,\bs,\Ascr}=\rho(\kappa(\bs(h)),\H)\in\mathsf{CC}(\bs,\H,\Ascr)$. 

Consider the loss $l^h(f,(x,y)):=\ind\{h(x)=y\neq f(x)\}$ for any $f\in\Y^\X$ and $(x,y)\in\Z$. Since $f_{h,S,\Ascr}$ realizes $S_1(h)$, we have $L^h_S(f_{h,S,\Ascr})=0$. 
Then, by \cite[Theorem 2.1]{NIPS2016_59f51fd6}, there exists some universal constant $C_1>0$ such that with probability at least $1-\delta$, 
\begin{align*}
P(\{(x,y)\in\Z:h(x)=y\neq f_{h,S,\Ascr}(x)\})=L_P^h(f_{h,S,\Ascr})\le \tfrac{C_1}{n}(\size_\Ascr(n)\log(n)+\log(1/\delta)).
\qquad \qedhere
\end{align*}
\end{proof}

To apply \cref{alg:CC}, we still need to construct a sample compression scheme for $\H$. It has been shown in the literature that a deterministic realizable learner can be converted to sample compression scheme whose size depends linearly on the realizable sample complexity of the learner under a constant (less than 1/2) error threshold \cite{freund1995boosting,NIPS2016_59f51fd6}. 
As we discussed after \cref{def:realizable_dim}, there exists a realizable learner called one-inclusion algorithm $\A_{\oig}$ with $\d_{\RE}(\A_{\oig},\H,1/3)=O(\d_\RE(\H))$. Using this learner, we can construct a sample compression scheme for $\H$ provided that $\d_{\RE}(\H)<\infty$. For completeness, we include the construction algorithm as well as the proof of the following result in \cref{sec:compression}, where we follow the boosting technique summarized in \cite[Section C]{NIPS2016_59f51fd6}. 
\begin{proposition} \label{prop:sample_comp_real}
Let $\H\subseteq\Y^\X$ have finite DS dimension so that $d:=\d_{\RE}(\H)$ is also finite. Then, there exists a sample compression scheme $\Ascr$ for $\H$ of size $\size_\Ascr(n)=\Theta(d\log(n))$ for $n\in\N$. 
Algorithmically, $\Ascr=\mathsf{SCSR}(\H,\A_{\oig})$ with $\mathsf{SCSR}$ defined in \cref{alg:sample_comp_real} and $\A_{\oig}$ being a deterministic learner called one-inclusion algorithm defined in \cref{def:oig}. 
\end{proposition}

%It follows that
%\begin{align}
%\notag &\er_{P}(f_{h,S_1})-\er_P(h)
%\\=& \notag
%P(\{(x,y)\in\Z:f_{h,S_1}(x)\neq y=h(x)\})-P(\{(x,y)\in\X\times\Y:h(x)\neq y=f_{h,S_1}(x)\})
%\\ \le & \notag
%P(\{(x,y)\in\Z:f_{h,S_1}(x)\neq y=h(x)\})
%\\ \le & \notag
%C_1\frac{k_1\log(n_1)+\log(1/\delta)}{n_1}
%\end{align}
%with probability at least $1-\delta$. 

\subsection{Learning a list} \label{sec:main2}

In this section, we provide \cref{alg:mul_weights} which learns a list of classifiers from a finite hypothesis class, say $\F$, such that the size of the learned list is only logarithmic in $|\F|$, and the list covers the correct predictions of each classifier in $\F$ up to additive small error. This is the key building block of our agnostic learner as it dramatically decreases the size of the effective label space while preserving the correct prediction of each classifiers. Our algorithm can be viewed as a multiplicative weight algorithm \cite{littlestone:94} with an adaptive reward tailored for multiclass learning. 

\begin{algorithm}[H]
\caption{$\mathrm{Multiplicative\ weights}\ \MW(T,((x_t,y_t))_{t\in[T]},\eta,\F)$} \label{alg:mul_weights}
\SetKwInput{KwInput}{Input}
\SetKwInput{KwOutput}{Output}
\SetKwComment{Comment}{/* }{ */}
\KwInput{Number of rounds $T\in\N$, sequence of examples $((x_t,y_t))_{t\in[T]}\in(\X\times\Y)^T$, step size $\eta\in(0,1]$, finite set $\F\subseteq\Y^{\X}$.}
\KwOutput{A list of classifiers $\mu=\{h_1,\dots,h_{T-1}\}\subseteq\F$.}
Initialize $w_1=(w_{1,h})_{h\in\H}$ with $w_{1,h}\gets 1$\;
\For{$t\in[T]$}{
    Let $p_{t}\gets w_{t}/\sum_{h\in\H}w_{t,h}$\;
    Sample $h_t\sim p_t$\; % independently\;
    Define the reward function $r_t:\Y^{\X}\to[0,1]$, $h\mapsto\ind{\{h(x_t) = y_t \land h_s(x_t) \neq y_t\  \forall s\in[t-1]\}}$\;
    % \textcolor{red}{More direct to use rewards:
    % \begin{align*}
    %     r_t(h;x_t,y_t) = \ind{\{h(x_t) = y_t \land h_s(x_t) \neq y_t \quad \forall s < t\}}
    % \end{align*}
    % }
    \For{$h\in\F$}{
        Let $w_{t+1,h}\gets w_{t,h}e^{\eta r_t(h)}$\;
    }
}
\end{algorithm}
% \qian{
% We can define a family of reward functions
% \begin{align*}
% r_t^{(q)}(h):=\frac{\ind\{h(x_t)=y_t\}}{(1+\sum_{s=1}^{t-1}\ind\{h_s(x_t)=y_t\})^q},\ q\in\N\cup\{\infty\}
% \end{align*}
% so that $r_t^{(1)}$ corresponds to the initial reward function considered, and $r_t^{(\infty)}$ corresponds to the reward function $\ind\{h(x_t)=y_t\land h_s(x_t)\neq y_t,\ \forall s\in[t-1]\}$. 
% Then, the final bound will be 
% \begin{align*}
% \E_{(X,Y)\sim P}\left[\frac{\ind\{c(X)=Y\}}{\frac{1}{T}\big(1+\sum_{s=1}^{T-1}\ind\{h_s(X)=Y\}\big)^p}\mid h_1,\dots,h_{T-1}\right]
% \le O\left(\log(|\H|)+\log(1/\delta)+\sum_{t=1}^T\frac{1}{t^q}\right).
% \end{align*}
% which covers both results. 
% }

\begin{theorem} \label{thm:mul_weights}
Let $P$ be a probability distribution over $\X\times\Y$. Let the input sequence of examples be $(X_t,Y_t)_{t\in[T]}\sim P^T$ and $\eta=1/2$ in \cref{alg:mul_weights}. Then, for any $f^\star \in\F$, with probability at least $1-\delta$, for $\mu=\{h_1,\dots,h_{T-1}\}$, we have
% \begin{align*}
% \E_{(X,Y)\sim P}\left[\ind\{c(X)=Y \land h_s(X) \neq Y \quad \forall t < T\}\frac{1}{T}+\frac{1}{T}\big(\sum_{s=1}^{T-1}\ind\{h_s(X)=Y\}\big)^2\mid h_1,\dots,h_{T-1}\right] \le 
% 4\log(|\H|)+14\log(3/\delta)+27.
% \end{align*}
% for $\mu=\{h_1,\dots,h_{T-1}\}$, the above inequality implies that
\begin{align*}
\P_{(X,Y)\sim P}\left(f^\star(X)=Y \land f^\star(X) \notin\mu(X)\mid\mu\right)\le (4\log(|\F|)+14\log(3/\delta)+12)/T.
\end{align*}
Thus, $T\ge \left(4\log(|\F|)+14\log(3/\delta)+12\right)/\eps$ implies that with probability at least $1-\delta$, 
\[
\P_{(X,Y)\sim P}\left(f^\star (X)=Y \land f^\star \notin\mu(X)\mid\mu\right)\le\eps.
\] 
\end{theorem}
\begin{proof}
Using the generic MW regret guarantee in \cref{lem:mw-generic}, the following holds (with probability $1$):
\begin{align} \label{eq:alg_bound}
\sum_{t=1}^T\sum_{h\in\F}p_{t,h} r_t(h)
\ge \frac{1}{1+\eta}\sum_{t=1}^Tr_t(f^\star)-\frac{1}{\eta(1+\eta)}\log(|\F|).
\end{align}

Now, let the input sequence of examples in \cref{alg:mul_weights} be $(X_t,Y_t)_{t\in[T]}\sim P^T$. 
Define $\F_0:=\sigma(\emptyset)$ and $\F_t:=\sigma(((X_s,Y_s,h_s))_{s\in[t]})$ for all $t\in[T]$. Then, $(\F_t)$ is a filtration and $r_t(f^\star)$ is $\F_t$-measurable. 
By \cref{lem:cond_chernoff}, we have
\begin{align*}
\sum_{t=1}^Tr_t(c)\ge \frac{1}{2}\sum_{t=1}^T\E[r_t(f^\star)\mid\F_{t-1}]-\log(1/\delta)
\end{align*}
with probability at least $1-\delta$. Let $(X,Y)\sim P$ be independent of $\F_T$. Then, we have 
\begin{align*}
\sum_{t=1}^T\E[r_t(f^\star)\mid\F_{t-1}]=
&\sum_{t=1}^T\E\left[\ind\{f^\star(X_t)=Y_t\land h_s(X_t)\neq Y_t\ \forall s\in[t-1]\}\mid\F_{t-1}\right]
\\ = &
\sum_{t=1}^T\E\left[\ind\{f^\star(X)=Y\land h_s(X)\neq Y\ \forall s\in[t-1]\}\mid\F_{T}\right]
\\ \ge &
T\; \E\left[\ind\{f^\star(X)=Y\land h_s(X)\neq Y\ \forall s\in[T-1]\}\mid\F_{T}\right],
\end{align*}
which together with the inequality above implies that
\begin{align} \label{eq:rt_lb}
\sum_{t=1}^Tr_t(f^\star)\ge 
\frac{T}{2}\P(f^\star(X)=Y\land h_s(X)\neq Y\ \forall s\in[t-1]\mid\F_{T})
-\log(1/\delta)
\end{align}
with probability at least $1-\delta$.

Define $\F_0':=\sigma(\emptyset)$ and $\F_t':=\sigma(((X_s,Y_s))_{s\in[T]},(h_s)_{s\in[t]})$ for all $t\in[T]$. Then, $(\F_t')$ is a filtration and $r_t(h_t)$ is $\F_t'$-measurable. 
By \cref{lem:cond_chernoff}, we have
\begin{align*}
\sum_{t=1}^Tr_t(h_t)\ge \frac{1}{2}\sum_{t=1}^T\E[r_t(h_t)\mid\F_{t-1}']-\log(1/\delta)
=\frac{1}{2}\sum_{t=1}^T\sum_{h\in\F}p_{t,h}r_t(h)-\log(1/\delta)
\end{align*}
with probability at least $1-\delta$, which implies that
\begin{align}
\notag
\sum_{t=1}^T\sum_{h\in\F}p_{t,h}r_t(h)
\le 2\sum_{t=1}^Tr_t(h_t)+2\log(1/\delta).
\end{align}
Now, consider the filtration $(\F_t'')_{t=0}^{T-1}$ with $\F_t'':=\sigma((h_s)_{s\in[\min\{T,t+1\}]},((X_s,Y_s))_{s\in[t]})$ for $t\in[T]$ and $\F_0:=\sigma(h_1)$. 
Then, $(r_t(h_t))$ is $(\F_{t}'')$-adapted. By \cref{lem:cond_chernoff}, we have that with probability at least $1-\delta$, 
\begin{align*}
\sum_{t=1}^Tr_t(h_t)\le &2\sum_{t=1}^T\E[r_t(h_t)\mid\F_{t-1}'']+\log(1/\delta)
\\ = &
2\sum_{t=1}^T\E\left[\ind\{h_t(X_t)=Y_t\land h_s(X_t)\neq Y_t\ \forall s\in[t-1]\}\mid\F_{t-1}''\right]+\log(1/\delta)
\\ = &
2\E\left[\sum_{t=1}^T\ind\{h_t(X)=Y\land h_s(X)\neq Y\ \forall s\in[t-1]\}\mid h_1,\dots,h_T\right]+\log(1/\delta)
\\ \le &
2+\log(1/\delta).
\end{align*}
Using a union bound, we have that with probability at least $1-2\delta$, 
\begin{align} \label{eq:rt_ub}
\sum_{t=1}^T\sum_{h\in\F}p_{t,h}r_t(h)
\le 2\sum_{t=1}^Tr_t(h_t)+2\log(1/\delta)
\le 
4+4\log(1/\delta).
\end{align}

For $\eta=1/2$, \eqref{eq:alg_bound}, \eqref{eq:rt_lb}, and \eqref{eq:rt_ub} imply that
\begin{align*}
4+4\log(3/\delta)
\ge &
\frac{2}{3}\sum_{t=1}^Tr_t(c)-\frac{4}{3}\log(|\F|)
\\ \ge &
\frac{T}{3}\P(f^\star(X)=Y\land h_s(X)\neq Y\ \forall s\in[t-1]\mid\F_{T})
-\frac{2}{3}\log(3/\delta)-\frac{4}{3}\log(|\F|)
\end{align*}
and
\begin{align*}
T\P(f^\star(X)=Y\land h_s(X)\neq Y\ \forall s\in[t-1]\mid\F_{T}) \le 
4\log(|\F|)+14\log(3/\delta)+12
\end{align*}
with probability at least $1-\delta$. 
Define the menu $\mu:=\{h_1,\dots,h_{T-1}\}$. Then, for any $f^\star \in\H$, we have
\begin{align*}
\E\left[\ind\{f^\star(X)=Y \land f^\star(X)\notin\mu(X)\}\mid\mu\right]\le \frac{4\log(|\F|)+14\log(3/\delta)+12}{T}
\end{align*}
with probability at least $1-\delta$.
$T\ge \frac{1}{\eps}\left(4\log(|\F|)+14\log(3/\delta)+12\right)$ implies that $B/T<\eps$. 
%By \cref{lem:logeineq}, $T\ge \frac{4}{\eps}\left(3\log(6/\eps)+2\log(|\H|)+4\log(2/\delta)+6\right)$ implies that $B/T<\eps$. 
\end{proof}

\subsection{Learning a classifier from the list} \label{sec:main3}
In this section, we learn a classifier from a finite list of classifiers which are used to reduce the effective label space. Formally, a list of classifier can be viewed as a ``menu'' \cite[Definition 30]{brukhim2022characterization} which are functions that map from $\X$ to $\power^\Y$, the power set of $\Y$. The size of the menu is the maximum size of its predicted size. 

Now, suppose that we are given a menu $\mu:X\to\power^\Y$. Our learner is based on a selection scheme which is consistent with $\H$ on samples in the correct region of $\mu$. 
Such selection schemes can be conveniently viewed as sample compression schemes under a special loss function: consider the loss function $l^\mu(f,(x,y)):=\ind\{f(x)\neq y\in\mu(x)\}$ defined for any classifier $f\in\Y^\X$ and pair $(x,y)\in\Z$. 
Our theoretical result is summarized below. 
\begin{theorem} \label{thm:step3}
Let $\mu$ be a menu and $\Ascr(\cdot,\H,\mu)$ be an $l^\mu$-sample compression scheme for $\H$ of size $\size_\Ascr(n)$ for $n\in\N$.
For any distribution $P$ over $\Z$, sample $S\sim P^n$ and let $\hat h_{\Ascr}:=\Ascr(S,\H,\mu)$. 
Then, there exists some universal constant $C_2>0$ such that for any $h\in\H$, with probability at least $1-\delta$ for $\delta\in(0,1)$, we have
\begin{align*}
P(\{(x,y)\in\Z:\hat h_\Ascr(x)\neq y\in\mu(x)\})\le 
\inf_{h'\in\H}L_P^{\mu}(h')+C_2\sqrt{\frac{\size_\Ascr(n)\log(n)+\log(1/\delta)}{n}}
\end{align*}
and
\begin{align*}
\er_P(\hat h_\Ascr)-\er_P(h)
\le 
C_2\sqrt{\frac{\size_\Ascr(n)\log(n)+\log(1/\delta)}{n}}
+P(\{(x,y)\in\Z:h(x)=y\notin\mu(x)\}).
\end{align*}
\end{theorem}
\begin{proof}
Since $\Ascr$ is an $l^\mu$-sample compression scheme for $\H$, by \cite[Theorem 3.1]{NIPS2016_59f51fd6}, there exists a constant $C_2>0$ such that with probability at least $1-\delta$, 
\begin{align}
\notag
P(\{(x,y)\in\Z:\hat h_\Ascr(x)\neq y\in\mu(x)\})=
L_P^{\mu}(\hat h_\Ascr)
\le & \notag 
\inf_{h'\in\H}L_P^{\mu}(h')+C_2\sqrt{\frac{\size_\Ascr(n)\log(n)+\log(1/\delta)}{n}}.
\end{align}
Since $h\in\H$, we immediately have
\begin{align*}
&P(\{(x,y)\in\Z:\hat h_\Ascr(x)\neq y\in\mu(x)\}) 
\\ \le &
P(\{(x,y)\in\Z:h(x)\neq y\in\mu(x)\})+C_2\sqrt{\frac{\size_\Ascr(n)\log(n)+\log(1/\delta)}{n}}.
\end{align*}
It follows from the above inequality that 
\begin{align*}
&\er_P(\hat h_\Ascr)-\er_P(h)
\\=&
P(\{(x,y)\in\Z:\hat h_\Ascr(x)\neq y=h(x)\})-P(\{(x,y)\in\Z:h(x)\neq y=\hat h_\Ascr(x)\})
\\ \le &
P(\{(x,y)\in\Z:\hat h_\Ascr(x)\neq y=h(x)\in\mu(x)\})
-P(\{(x,y)\in\Z:h(x)\neq y=\hat h_\Ascr(x)\in\mu(x)\})
\\&+P(\{(x,y)\in\Z:\hat h_\Ascr(x)\neq y=h(x)\notin\mu(x)\})
\\ 
\le &
P(\{(x,y)\in\Z:\hat h_\Ascr(x)\neq y\in\mu(x)\})-P(\{(x,y)\in\Z:h(x)\neq y\in\mu(x)\})
\\&+P(\{(x,y)\in\Z:h(x)=y\notin\mu(x)\})
\\ 
\le &
C_2\sqrt{\frac{\size_\Ascr(n)\log(n)+\log(1/\delta)}{n}}
+P(\{(x,y)\in\Z:h(x)=y\notin\mu(x)\}).
\qquad \qedhere
\end{align*}
\end{proof}

Now, we are left with constructing an $l^\mu$-sample compression scheme for $\H$. Following the standard construction based on realizable learner in \cite[Section C]{NIPS2016_59f51fd6}, we would like to develop an algorithm whose expected risk under $l^\mu$ is bounded by a constant less than 1/2 when trained on finite samples. Since $l^\mu$ counts the classification errors for those points realizable by $\mu$, we can restrict $\H$ to those classifiers consistent with $\mu$ on the given samples, which reduces the size of the effective label space to the size of $\mu$. Then, the problem reduces to multiclass learning with finite effective label space, say of size $K$, for which multiple learners including the one-inclusion algorithm $\A_{\oig}$ have been shown to achieve expected error rate $d_N\log(K)/n$ trained on $n$ samples, where $d_N$ is the Natarajan dimension. 

However, there still exists a technical caveat in the procedure described above. To construct the restricted hypothesis class, we need to observe all training samples, which is prohibited for the reconstruction map. To resolve this challenge, we leverage the theory of partial concept class \cite{partial_concept} to define a partial hypothesis class by modifying the classifiers in $\H$ to be ``undefined'' on regions where they disagree with $\mu$. In this way, we are able to reduce the effective label space without observing the samples. Then, generalizing the PAC learning analysis in \cite{partial_concept} for partial concept class of finite VC dimension, we can convert the one-inclusion algorithm to a realizable PAC learner for partial hypothesis class of finite Natarajan dimension and effective label space, which completes the picture of the $l^\mu$-sample compression scheme. 
Moreover, it is worth pointing out that the partial hypothesis class mentioned above is only constructed for the proof. Algorithmically, we can restrict to classifiers in $\H$ that are consistent with $\mu$ as specified in \cref{alg:from_list}. 

Our result on constructing $l^\mu$-sample compression scheme is summarized in the following proposition. Its proof can be found in \cref{sec:partial_sample_compression} where we also include the concepts of partial hypotheses.  
\begin{proposition} \label{prop:list_realizable_compression}
Let $\mu:\X\to\power^\Y$ be a menu of size $p\in\N$ and $\H\subseteq\Y^\X$ be a hypothesis class of Natarajan dimension $d_N\in\N$. Then, there exists an $l^\mu$-sample compression scheme $\Ascr=\mathsf{LSCS}(\mu,\H)$ for $\H$ with $\mathsf{LSCS}$ defined by \cref{alg:sample_comp_list} of size $\size_{\Ascr}(n)=\Theta(d_N\log(p)\log(n))$ for $n\in\N$. 
\end{proposition}

\subsection{Algorithm and proof of upper bound} \label{sec:main4}
In this section, we present our multiclass agnostic PAC learner $\learner$ defined by \cref{alg:main} which combines the three building blocks constructed in the previous subsections. 
For readers' convenience, the references of the algorithms called by $\learner$ are also provided here: 
$\mathsf{SCSR}$ is defined by \cref{alg:sample_comp_real}; 
$\A_{\oig}$ is the one-inclusion algorithm defined by \cref{alg:oia};
$\mathsf{CC}$ is defined by \cref{alg:CC}; 
$\mathsf{MW}$ is defined by \cref{alg:mul_weights}; 
$\mathsf{LSCS}$ is defined by \cref{alg:sample_comp_list}. 

\begin{algorithm}[H]
\caption{Multiclass agnostic PAC learner $\learner(\bs,\H)$} \label{alg:main}
\SetKwInput{KwInput}{Input}
\SetKwInput{KwOutput}{Output}
\SetKwComment{Comment}{/* }{ */}
\KwInput{Sequence $\bs=((x_1,y_1),\dots,(x_n,y_n))\in\Z^n$ with $n\ge3$, hypothesis class $\H\subseteq\Y^{\X}$ of finite DS dimension.}
\KwOutput{A classifier $\hat h$.}
Let $\bs_1\gets ((x_i,y_i)_{i=1}^{\lfloor n/3\rfloor})$, $\bs_2\gets ((x_i,y_i)_{i=\lfloor n/3\rfloor+1}^{2\lfloor n/3\rfloor})$, and $\bs_3\gets ((x_i,y_i)_{i=2\lfloor n/3\rfloor+1}^{n})$\;
Let $\Ascr_{1}\gets\mathsf{SCSR}(\H,\A_{\oig})$ and $\hat\F\gets\mathsf{CC}(\bs_1,\H,\Ascr_{1})$\;
Let $\hat\mu\gets\mathsf{MW}(\lfloor n/3\rfloor,\bs_2,1/2,\hat\F)$\;
Let $\Ascr_{2}\gets\mathsf{LSCS}(\hat\mu,\H)$ and $\hat h\gets\Ascr_{2}(\bs_3)$\;
\end{algorithm}

Now, we are ready to prove the upper bound in \cref{thm:main}. As we mentioned previously, the proof of the lower bound is provided in \cref{sec:lb}. 
\begin{proof}[Proof of upper bounds in \cref{thm:main}]
Let $P$ be an arbitrary distribution over $\Z=\X\times\Y$ and $\opt\in\H$ be such that $\er_P(\opt)\le \inf_{h\in\H}\er_P(h)+\eps/4$. 

For $n_1\in\N$ to be specified later, sample $S_1\sim P^{n_1}$. 
Let $\Ascr_1:=\mathsf{SCSR}(\H,\A_{\oig})$ and $\hat\F:=\mathsf{CC}(S_1,\H,\Ascr_1)$. 
By \cref{prop:sample_comp_real,thm:CC}, there exists a positive integer $k_1=O(d\log(n_1))$ such that  $|\hat\F|\le n_1^{k_1}$ and for each $h\in\H$, with probability at least $1-\delta/3$, there exists $f_h\in\hat\F$ such that
\begin{align} \label{eq:bound1}
P(\{(x,y)\in\Z:h(x)=y\neq f_{h}(x)\})\le C_1\frac{k_1\log(n_1)+\log(3/\delta)}{n_1}. 
\end{align}

Define the positive integer 
\begin{align*}
n_2:=\lceil8\left(2\log(|\F(S_1,\H)|)+7\log(9/\delta)+6\right)/\eps\rceil
=O(\log(n_1^{k_1}/\delta)/\eps)
\end{align*}
and sample $S_2\sim P^{n_2}$ independently. 
Let $\hat\mu:=\MW(n_2,S_2,1/2,\hat\F)$ be the output menu of \cref{alg:mul_weights}. Then, letting $c=f_{\opt}$ in \cref{thm:mul_weights}, we have 
\begin{align*}
P(\{(x,y)\in\Z:f_{\opt}(x)=y\notin\hat\mu(x)\})<\eps/4
\end{align*}
with probability at least $1-\delta/3$ conditional on $\hat\F$. Note that the size of $\hat\mu$ is $n_2-1$. 
By \eqref{eq:bound1} and union bound, we have
\begin{align}
\notag
&P(\{(x,y)\in\Z:\opt(x)=y\notin\hat\mu(x)\})
\\ \le & \notag
P(\{(x,y)\in\Z:\opt(x)=y\neq f_{\opt}(x)\})
+P(\{(x,y)\in\Z:f_{\opt}(x)=y\notin\hat\mu(x)\})
\\ < & \label{eq:list_cover}
C_1\frac{k_1\log(n_1)+\log(3/\delta)}{n_1}+\eps/4
\end{align}
with probability at least $1-2\delta/3$. 

Sample $S_3\sim P^{n_3}$ independently with $n_3\in\N$ to be specified later. 
Let $\Ascr_{2}:=\mathsf{LSCS}(\hat\mu,\H)$ and $\hat h:=\Ascr_{2}(S_3)$. 
By \cref{prop:list_realizable_compression,thm:step3}, there exists some positive integer $k_2=O(d_N\log(n_2)\log(n_3))$ such that 
\begin{align} \label{eq:list_best}
\er_P(\hat h)-\er_P(\opt)
\le 
C_2\sqrt{\frac{k_2\log(n_3)+\log(3/\delta)}{n_3}}
+P(\{(x,y)\in\Z:\opt(x)=y\notin\hat\mu(x)\})
\end{align}
with probability at least $1-\delta/3$. 

It follows from \eqref{eq:list_cover}, \eqref{eq:list_best}, and union bound that 
\begin{align*}
\er_P(\hat h)-\er_P(\opt)
\le &
C_1\frac{k_1\log(n_1)+\log(3/\delta)}{n_1}+C_2\sqrt{\frac{k_2\log(n_3)+\log(3/\delta)}{n_3}}
+\eps/4
\end{align*}
with probability at least $1-\delta$. 
Since $\er_P(\opt)\le \inf_{h\in\H}\er_P(h)+\eps/4$, we can conclude that 
\begin{align} \label{eq:bound2}
\er_P(\hat h)-\inf_{h\in\H}\er_P(h)
\le 
\eps/2
+C_1\frac{k_1\log(n_1)+\log(3/\delta)}{n_1}
+C_2\sqrt{\frac{k_2\log(n_3)+\log(3/\delta)}{n_3}}
\end{align}
with probability at least $1-\delta$. 
% Also note that $\hat h=\learner((S_1,S_2,S_3),\H)$ for $\learner$ defined in \cref{alg:main}. 

Since $k_1=O(d\log(n_1))$, by \cref{lem:log_ineq}, there exists 
\begin{align*}
n_1=\Theta\left(\frac{d\log^2(d/\eps)+\log(1/\delta)}{\eps}\right)
\end{align*}
such that $C_1\frac{k_1\log(n_1)+\log(3/\delta)}{n_1}\le\eps/4$. 
Since $n_2=O(\log(n_1^{k_1}/\delta)/\eps)$, we have
\begin{align*}
&\log(n_1)=\Theta(\log((d+\log(1/\delta))/\eps)),\\
&n_2=
O\left(\frac{d\log^2(n_1)+\log(1/\delta)}{\eps}\right)=
O\left(\frac{d\log^2((d+\log(1/\delta))/\eps)+\log(1/\delta)}{\eps}\right),\textup{ and}\\
&\log(n_2)=O\left(\log((d+\log(1/\delta))/\eps)\right).
\end{align*}
Since $k_2=O(d_N\log(n_2)\log(n_3))$, by \cref{lem:log_ineq}, there exists 
\begin{align*}
n_3=&\Theta\left(\frac{d_N\log(n_2)\log^2(d_N\log(n_2)/\eps)+\log(1/\delta)}{\eps^2}\right)
\\=&
\Theta\left(\frac{d_N}{\eps^2}
\log\left(\frac{d+\log(1/\delta)}{\eps}\right)
\log^2\left(\frac{d_N}{\eps}\log\left(\frac{d+\log(1/\delta)}{\eps}\right)\right)
+\frac{\log(1/\delta)}{\eps^2}\right)
\end{align*}
such that $C_2\sqrt{\frac{k_2\log(n_3)+\log(3/\delta)}{n_3}}\le\eps/4$. 

With the above settings of $n_1,n_2,n_3$, we have $\er_P(\hat h)\le \inf_{h\in\H}\er_P(h)+\eps$ with probability at least $1-\delta$ by \eqref{eq:bound2}. 
Noticing that all of our guarantees hold when $n_1,n_2,n_3$ taking values even larger than the above settings, 
there exists 
$n_{\eps,\delta}=\Theta(n_1+n_2+n_3)$ such that 
$$
\er_P(\learner(S,\H))\le \textstyle\inf_{h\in\H}\er_P(h)+\eps
$$ 
with probability at least $1-\delta$ for $S\sim P^n$ with $n\ge n_{\eps,\delta}$. 

Since $P$ is arbitrary, the agnostic sample complexity of $\learner$ is upper bounded by $n_{\eps,\delta}$ with
\begin{align} \notag
n_{\eps,\delta}=O\bigg(&\frac{d_N}{\eps^2}
\log\bigg(\frac{d+\log(1/\delta)}{\eps}\bigg)
\log^2\bigg(\frac{d_N}{\eps}\log\bigg(\frac{d+\log(1/\delta)}{\eps}\bigg)\bigg)
\\& \label{eq:ub_raw}
+\frac{d}{\eps}\log^2\bigg(\frac{d+\log(1/\delta)}{\eps}\bigg)
+\frac{\log(1/\delta)}{\eps^2}
\bigg).
\end{align}
By comparing different terms in the above bound and using the relation that $d_N\le d_{DS}$ as well as $\Omega(d_{DS})\le d\le O(d_{DS}^{3/2}\log^2(d_{DS}))$, we can simplify the bound in \eqref{eq:ub_raw} and show that 
\begin{align*}
n_{\eps,\delta}=O\bigg(&\frac{d_N\log^3(d_N/\eps)+\log(1/\delta)}{\eps^2}
+\frac{d\log^2(d/\eps)}{\eps}\bigg). 
\end{align*}
The upper bound in \eqref{eq:current_ub} can be obtained in a similar manner by plugging $d=\Theta(d_{DS}^{3/2}\log^2(d_{DS}))$ into the upper bound in \eqref{eq:main}. Since such derivations involve tedious calculations, we defer the simplification of those upper bounds to \cref{lem:simp} in \cref{sec:lemmas}. 
\end{proof}

\section*{Acknowledgments}

This project has received funding from the European Research Council (ERC) under the European
Union’s Horizon 2020 research and innovation program (grant agreements No.\ 882396 and No.\ 101078075). Views and opinions expressed are however those of the author(s) only
and do not necessarily reflect those of the European Union or the European Research Council. Neither
the European Union nor the granting authority can be held responsible for them. This work received
additional support from the Israel Science Foundation (ISF, grant numbers 2250/22, 993/17 and 3174/23), Tel Aviv University Center for AI and Data Science (TAD), and the Len Blavatnik and the Blavatnik Family Foundation. Shay Moran is a Robert J.\ Shillman Fellow; he acknowledges support by ISF grant 1225/20, by BSF grant 2018385, by Israel PBC-VATAT, by the Technion Center for Machine Learning and Intelligent Systems (MLIS), and by the the European Union (ERC, GENERALIZATION, 101039692). Views and opinions expressed are however those of the author(s) only and do not necessarily reflect those of the European Union or the European Research Council Executive Agency. Neither the European Union nor the granting authority can be held responsible for them.

\newpage
\appendix

\section{Local guarantee for list learning - an example}
\label{sec:example}
We illustrate here why, in the label-space reduction stage, it does \emph{not} suffice to learn a list satisfying the standard regret-type guarantee
\[
\P\bigl[Y \notin \text{list}(X)\bigr]
\;\le\; \er_P(h^\star) + \epsilon,
\]
where \(h^\star\) is the optimal hypothesis in \(\mathcal{H}\).
As discussed in the technical overview, the requirement we use is the more localized condition that
\[
h^\star(X)=Y \;\Longrightarrow\; h^\star(X)\in \text{list}(X)
\qquad\text{except with probability }O(\varepsilon)\text{ over }(X,Y)\sim P,
\]
which ensures that whenever \(h^\star\) predicts correctly, its prediction also appears in the list.

Consider the following example.
Let the instance space be \(\mathcal{X}=[0,1]\) and the label space be \([K]=\{0,1,\ldots,K-1\}\) for a large integer \(K\).
Let
\[
\mathcal{H}=\{h_i : h_i(x)=i,\; 0\le i\le K-1\}
\]
be the class of constant classifiers.
Define a distribution \(P\) by taking \(X\) uniformly distributed over \([0,1]\), and setting the conditional distribution of \(Y\) as
\[
Y \mid (X=x) \;:=\;
\begin{cases}
0, & x \le 1/3,\\[6pt]
\text{uniform over }\{1,\ldots,K-1\}, & x>1/3.
\end{cases}
\]
The optimal hypothesis is \(h_0\), whose error is
\(
\er_P(h_0)=\P[X>1/3]=2/3.
\)

Now take the list
\[
\text{list}(X)=\{1,\ldots,K-1\}
\quad\text{for all }X.
\]
Then
\[
\P[Y\notin\text{list}(X)]
=\P[X\le 1/3]
=\frac{1}{3}
\;\le\; \er_P(h^\star),
\]
so the usual regret-style condition is satisfied.

However, any classifier constrained to predict only from this list must misclassify every \(x\le 1/3\), contributing error \(1/3\).
On \(x>1/3\), the label is uniform over \(\{1,\ldots,K-1\}\), so any fixed prediction among these labels is correct with probability \(1/(K-1)\).
Thus the total error of any list-bounded classifier is at least
\[
\frac{1}{3}
+\frac{2}{3}\Bigl(1-\frac{1}{K-1}\Bigr)
=1-O\!\left(\frac{1}{K}\right),
\]
which exceeds the optimal error \(2/3\) by almost \(1/3\).

This example shows that the standard condition
\(\P[Y\notin\text{list}(X)] \le \operatorname{er}_P(h^\star)+\epsilon\)
is insufficient for the last stage of the algorithm: even if the list contains the true label with the optimal frequency, restricting predictions to the list can incur much larger error.
The stronger localized guarantee is therefore essential.

\section{Discussion on the realizable dimension} \label{sec:discuss_dim}
%\qian{For any learner $\A$ and any $r\in(0,1/2)$ such that $\d_{\RE}(\A,\H,r)<\infty$, by \cref{lem:selection_margin}, there exists a sample compression scheme $\Ascr$ based on $\A$ of size $k(n)=\lceil4\log(n+1)/(1/2r-1)^2r\rceil\d_{\RE}(\A,\H,r)$ for $n\in\N$. Then, by \cite[Theorem 3.1]{NIPS2016_59f51fd6}, we have 
%\begin{align*}
%\eps_{\Ascr,\H}^\RE(n)\le C_3\frac{k(n)\log n}{n}
%\end{align*}
%for some universal constants $C_3>0$. For any $r'\in(0,r)$, by \cref{lem:log_ineq}, 
%\begin{align*}
%\frac{C_3d}{(1/2r-1)^2r}\frac{\log^2(n)}{n}\le r',\ \frac{C_3d}{(1/2r-1)^2rr'}\log^2\frac{C_3d}{(1/2r-1)^2rr'}
%\end{align*}}

As we mentioned in \cref{sec:dim}, \eqref{eq:real_dim_bound0} follows directly from \cite[Theorem 1.9]{hanneke2024improved} by setting $\eps=\delta=r/2$ and noting that $r$ is a positive constant. 

To further understand the relationship between $\d_{\DS}(\cdot)$, $\d_\RE(\cdot,r)$, and $\d^{\det}_\RE(\cdot,r)$, we will introduce a deterministic multiclass learner called ``one-inclusion algorithm'' \cite{NIPS2006_a11ce019,daniely2014optimal,brukhim2022characterization} whose expected error rate can be upper bounded by the expected error rate of any multiclass learner up to constant factors. To define the one-inclusion algorithm, as its name indicates, we first define ``one-inclusion graph'' which is a hypergraph structure on a finite set of label sequences. 
\begin{definition}[One-inclusion graph, \cite{haussler:94}] \label{def:oig}
The \textbf{one-inclusion graph} (OIG) of $V\subseteq\Y^n$ for $n\in\N$ is a hypergraph $\G(H)=(V,E)$ where $V$ is the vertex-set and $E$ denotes the edge-set defined as follows. For any $i\in[n]$ and $f:[n]\backslash\{i\}\to\Y$, we define the set $e_{i,f}:=\{h\in V:h(j)=f(j),\ \forall j\in[n]\backslash\{i\}\}$.
Then, the edge-set is defined as
\begin{align*}
E:=\{(e_{i,f},i):i\in[n],f:[n]\backslash\{i\}\rightarrow\Y,e_{i,f}\neq\emptyset\}.
\end{align*}
For any $(e_{i,f},i)\in E$ and $h\in V$, we say $h\in(e_{i,f},i)$ if $h\in e_{i,f}$ and the size of the edge is $|(e_{i,f},i)|:=|e_{i,f}|$. 
The \textbf{direction} of $e=(e_{i,f},i)$ is defined as $\dir(e)=i$. 
\end{definition}
One-inclusion graph can be defined on the projection of a hypothesis class $\H$ on an input sequence $\bx\in\X^n$ with $n\in\N$, which gives $\G(\H|_{\bx})$ following our notation. 
With $\G(\H|_{\bx})$, a classifier can be constructed via ``orienting'' the edges in $\G(\H|_{\bx})$ to its vertices, which is the intuition of the one-inclusion algorithm. Formally, an orientation of a hypergraph is defined below. 
\begin{definition}[Orientation, \cite{daniely2014optimal}]
An \textbf{orientation} of a hypergraph $(V,E)$ is a mapping $\sigma:E\to V$ such that $\sigma(e)\in e$ for each edge $e\in E$. 
The \textbf{out-degree} of a vertex $v\in V$ is $\outdeg(v;\sigma):=|\{e\in E:v\in e\textup{ and }\sigma(e)\neq v\}|$ and the \textbf{maximum out-degree} of $\sigma$ is $\outdeg(\sigma):=\sup_{v\in V}\outdeg(v;\sigma)$. 
\end{definition}
Now, we are ready to present the one-inclusion algorithm. 

\begin{algorithm}[H]
\caption{The one-inclusion algorithm $\A_{\oig}(\bs,\H)$ \cite{daniely2014optimal,brukhim2022characterization}} \label{alg:oia}
\SetKwInput{KwInput}{Input}
\SetKwInput{KwOutput}{Output}
\KwInput{Hypothesis class $\H\subseteq\Y^\X$, $\H$-realizable sequence $\bs=((x_1,y_1),\dots,(x_n,y_n))\in\Z^n$ with $n\in\N$.}
\KwOutput{A classifier $\A_{\oig}(\bs,\H)=h_{\bs}\in {\Y}^{\X}$.}
Given $x\in\X$, the value of $h_{\bs}(x)\in\Y$ is calculated as follows:
\begin{enumerate}[label=\roman*.]
    \item Consider $\H|_{(x_1,\dots,x_n,x)}\subseteq\Y^{n+1}$;
    \item Find an orientation $\sigma$ of $\G(\H|_{(x_1,\dots,x_n,x)})$ that minimizes the maximum out-degree;
    \item Let $e\gets \{h\in \H|_{(x_1,\dots,x_n,x)}:h(i)=y_i,\ \forall i\in[n]\}$;
    \item Let $h_{\bs}(x)\gets \sigma((e,n+1))(n+1)$;
\end{enumerate}
\end{algorithm}
The expected error rate of $\A_{\oig}$ can be bounded via a complexity measure of $\H$ called the ``density'', which has been shown to characterize $\eps_\H^{\RE}$ and also been used to define $\d_{\DS}(\H)$ in \cite{daniely2014optimal}. 
Specifically, the density is defined via the ``maximal average degree'' of $\G(\cH|_{\bx})$. 
\begin{definition}[Degree and density, \cite{daniely2014optimal}] \label{def:deg}
For any hypergraph $G=(V,E)$ and $v\in V$, we define the 
\textbf{degree} of $v$ in $G$ to be $\deg(v;G):=|\{e\in E:v\in e,|e|\ge2\}|$.
% When the underlying graph is clear in the context, we simply write $\deg(v)$ in abbreviation. 
If $|V|<\infty$, we can define the \textbf{average degree} of $G$ to be 
\begin{align*}
\avgdeg(G):=\tfrac{1}{|V|}\textstyle\sum_{v\in V}\deg(v;G)=\tfrac{1}{|V|}\textstyle\sum_{e\in E:|e|\ge 2}|e|.
\end{align*}
For general $V$, we can define the \textbf{maximal average degree} of $G$ to be
$$
\md(G):=\textstyle\sup_{U\subseteq V:|U|<\infty}\avgdeg(G[U]),
$$
where $G[U]=(U,E[U])$ denotes the induced hypergraph of $G$ on $U\subseteq V$ with $E[U]:=\{e\cap U:e\in E,e\cap U\neq \emptyset\}$.

The \textbf{density} of $\cH\subseteq\Y^\X$ is defined as $\mu_{\cH}(m):=\textstyle\sup_{\bx\in\X^m}\md(\G(\cH|_{\bx})),\ \forall m\in\N$. 
\end{definition}
It is worth pointing out that the original definition of the ``density'' of a one-inclusion graph $G=(V,E)$ in \cite{haussler:94} corresponds to the average out-degree of $G$, $\avgoutdeg(G):=\frac{1}{|V|}\sum_{e\in E}(|e|-1)$, which leads to $\mu_{\H}'(m):=\sup_{\bx\in\X^m,U\subseteq \H|_{\bx}:|U|<\infty}\avgoutdeg(\G(\H|_{\bx})[U])$, c.f. \cite{10353103}.
It can be observed that $\avgoutdeg(G)\le \avgdeg(G)\le2\avgoutdeg(G)$ and hence $\mu'_{\H}\le \mu_{\H}\le 2\mu'_{\H}$. Since the two definitions of density differ only by a factor of 2, either will work for our purpose. 
Here, we follow \cite{daniely2014optimal} to name $\mu_\H$ density, where the DS dimension is first proposed and defined as $\d_{\DS}(\H)=\sup\{m\in\N\cup\{0\}:\mu_\H(m)=m\}$. 
It is not hard to see that this original definition is equivalent to the one in \cref{def:DS} via pseudo-cube. 

Now, we are ready to prove \eqref{eq:real_dim_bound}. 
\begin{proof}[Proof of \eqref{eq:real_dim_bound}]
Note that $\A_\oig$ in \cref{alg:oia} is a deterministic algorithm.  
By \cite[Lemma A.1]{10353103} and \cite[Proposition 15]{brukhim2022characterization}, we have $\eps_{\A_{\oig},\H}^\RE(n)\le \lceil\mu_\H(n+1)\rceil/(n+1)$, which implies that 
$$
\d_{\RE}(\H,r)\le \d_{\RE}(\A_{\oig},\H,r)\le \inf\{n\in\N:n+1\ge \lceil\mu_\H(n+1)\rceil/r\}.
$$
By \cite[Theorem 6]{10.1145/3564246.3585190}, we have $\eps_\H^\RE(n)\ge\mu_\H(n)/2en$, which implies that
\begin{align} \label{eq:real_lb}
\d_\RE(\H,r)\ge \inf\{n\in\N:n\ge \mu_{\H}(n)/2er\}
\end{align}
and $\eps_\H^\RE(n)\ge(\lceil\mu_\H(n)\rceil-1)/2en\ge (\eps^\RE_{\A_{\oig},\H}(n-1)-1/n)/2e$. It follows that
\begin{equation} \label{eq:oia}
\begin{aligned}
\begin{cases}
&\eps^\RE_{\A_{\oig},\H}(n)\le 2e\eps_\H^\RE(n+1)+1/(n+1) 
\textup{ and }
\\&
\d_\RE(\H,\A_\oig,r)\le \max\{\d_{\RE}(\H,r/3e), 3/r\}-1.
\end{cases}
\end{aligned}
\end{equation}
%Setting $r=1/3$, we have $\d_\RE(\H)\le \d_\RE(\H,\A_\oig,1/3)\le \max\{\d_{\RE}(\H,1/3e), 9\}-1$. 
Now, \eqref{eq:real_dim_bound} follows from \eqref{eq:real_dim_bound0} and the fact that $\d_{\RE}(\H,r)\le\d^{\det}_{\RE}(\H,r)$. 
\end{proof}

\section{Boosting and sample compression scheme} \label{sec:compression}
\begin{lemma} \label{lem:selection_margin}
Suppose that there exists a deterministic multiclass learner $\A$ such that for some constant $\gamma\in(0,1/2)$, there exists an integer $m_\gamma\in\N$ satisfying that for any distribution $Q$ over $\bs=((x_1,y_1),\dots,(x_n,y_n))\in(\X\times\Y)^n$, there exists some sequence $\bt\in \bs^{m_\gamma}$ with $\er_Q(\A(\bt))\le\gamma$. 

Then, for any $\eps\in(0,1/2\gamma-1)$ and any integer $a\ge\ell_\eps:=\lceil3\log(n+1)/\eps^2\gamma\rceil$, there exists a sample compression scheme $(\kappa,\rho)$ on $\bs$ of size $m_\gamma a$ such that $\kappa$ selects $(\bt^{(1)},\dots,\bt^{(a)})\in\bs^{m_\gamma a}$ with $\bt^{(j)}\in\bs^{m_{\gamma}}$, $\forall j\in[a]$, and $\rho$ outputs a classifier $\hat h=\rho(\bt^{(1)},\dots,\bt^{(a)}):=\Maj(\A(\bt^{(1)}),\dots,\A(\bt^{(a)}))$ with $\hat h(x)=y$ for all $(x,y)\in \bs$. 
\end{lemma}
\begin{proof}
For any discrete set $S$, let $\scr P(S)$ denote the set of all distributions over $S$. 
Let $m:=m_\gamma$ and $l:\Y^\X\times(\X\times\Y)\to\{0,1\}$, $(h,(x,y))\mapsto\ind\{h(x)\neq y\}$ be the zero-one loss. It follows from the assumption of the lemma that 
\begin{align*}
\max_{Q\in\scr P(\bs)}\min_{\bt\in \bs^{m}}\E_{(X,Y)\sim Q}[l(\A(\bt),(X,Y))]\le \gamma.
\end{align*}
By von Neumann's minimax theorem \cite{Neumann1928}, we have that
\begin{align*}
\min_{\D\in\scr P(\bs^m)}\max_{(x,y)\in \bs}\E_{T\sim \D}[l(\A(T),(x,y))]=
\max_{Q\in\scr P(\bs)}\min_{\bt\in \bs^m}\E_{(X,Y)\sim Q}[l(\A(\bt),(X,Y))]\le \gamma.
\end{align*}
Let $\D_*\in\scr P(\bs^m)$ minimize $\max_{(x,y)\in \bs}\E_{T\sim \D}[l(\A(T),(x,y))]$ over $\D\in\scr P(\bs^m)$. Then, we have
\begin{align*}
\max_{(x,y)\in\bs}\E_{T\sim\D_*}[l(\A(T),(x,y))]\le \gamma.
\end{align*}
For $a\in\N$, sample $(T^1,\dots,T^a)\sim \D^a$. By the multiplicative Chernoff bound, we have
\begin{align*}
\P\bigg(\sum_{b=1}^a l(\A(T^b),(x,y))\ge (1+\eps)\gamma a\bigg)\le e^{-\eps^2\gamma a/3} 
\end{align*}
for any $(x,y)\in \bs$. It follows that
\begin{align*}
\P\bigg(\frac{1}{a}\sum_{b=1}^a l(\A(T^b),(x,y))\le (1+\eps)\gamma,\ \forall (x,y)\in \bs\bigg)\ge 1-ne^{-\eps^2\gamma a/3}.
\end{align*}
For $a\ge\ell_\eps=\lceil3\log(n+1)/\eps^2\gamma\rceil$, we have $1-ne^{-\eps^2\gamma a/3}\ge1/(n+1)$ and by the probabilistic method, there exist $\bt^{(1)},\dots,\bt^{(\ell_\eps)}\in\bs^m$ such that
\begin{align*}
\frac{1}{a}\sum_{b=1}^{a} \ind\{\A(\bt^{(b)})(x)\neq y\}=\frac{1}{a}\sum_{b=1}^{a} l(\A(\bt^{(b)}),(x,y))\le (1+\eps)\gamma<\frac{1}{2},\ \forall (x,y)\in \bs.
\end{align*}
Then, for $\hat h:=\Maj(\A(\bt^{(1)}),\dots,\A(\bt^{(a)}))$, we have $\hat h(x)=y$ for all $(x,y)\in \bs$. 

Now, we can define the selection scheme $(\kappa,\rho)$ as follows. The selection map $\kappa$ selects the sequence $(\bt^1,\dots,\bt^{a})\in\bs^{m_\gamma a}$ of size $m_\gamma a$ from the input sequence $\bs$. 
The reconstruction map $\rho$ outputs  
$$
\rho((\bt^{(1)},\dots,\bt^{(a)})):=\hat h=\Maj(\A(\bt^{(1)}),\dots,\A(\bt^{(a)}))
$$ 
which satisfies $\hat h(x)=y$ for all $(x,y)\in \bs$. 
Thus, $(\kappa,\rho)$ is a sample compression scheme on $\bs$ of size $k=m_\gamma a$. 
\end{proof}

Summarizing the above result, we can construct the sample compression scheme via the following algorithm. 

\begin{algorithm}[H]
\caption{Sample compression scheme via a realizable learner $\mathsf{SCSR}(\H,\A)$} \label{alg:sample_comp_real}
\SetKwInput{KwInput}{Input}
\SetKwInput{KwOutput}{Output}
\KwInput{Hypothesis class $\H\subseteq\Y^\X$, deterministic multiclass learner $\A$ with $\d_{\RE}(\A,\H,\gamma)<\infty$ with $\gamma\in(0,1/2)$.}
\KwOutput{A sample compression scheme $(\kappa,\rho)$ for $\H$ of size $k(n)$.}
Let $d\gets \d_\RE(\A,\H,\gamma)$ and $k(n):=\lceil3.01\log(n+1)/(1/2\gamma-1)^2\gamma\rceil d$ for all $n\in\N$\;
Given $\bt=(z_1,\dots,z_n)\in\Z^n$ with $n\in\N$, $\rho(\bt)\in\Y^\X$ is calculated as follows:
\begin{enumerate}[label=\roman*.]
    \item Let $a\gets \lfloor n/d\rfloor$;
    \item If $a<1$: let $\rho(\bt)\gets h^\dagger$ for arbitrary fixed $h^\dagger\in\Y^\X$ selected a priori;

    Else: let $\bt^{(i)}\gets (z_{(i-1)d+1},\dots,z_{id})$ for each $i\in[a]$ and 
    $$\rho(\bt)\gets \Maj(\A(\bt^{(1)},\H),\dots,\A(\bt^{(a)},\H));
    $$
\end{enumerate}
Given $\bs\in\Z^n$ with $n\in\N$, $\kappa(\bs)$ is calculated as follows:
\begin{enumerate}[label=\roman*.]
    \item Let $\tilde h\in\argmax_{h\in\H}|\bs(h)|$ and $\tilde \bs\gets \bs(\tilde h)$;
    \item Find $\bt\in\tilde\bs^{k(n)}$ such that $\rho(\bt)$ realizes $\tilde\bs$ and let $\kappa(\bs)\gets \bt$. 
\end{enumerate}
\end{algorithm}
An immediate consequence of \cref{lem:selection_margin} provides the guarantee of \cref{alg:sample_comp_real}. 
\begin{corollary} \label{coro:selection_margin}
Let $\A$ be a deterministic multiclass learner such that $\d_\RE(\A,\H,\gamma)\in\N$ for some $\H\subseteq\Y^\X$ and constant $\gamma\in(0,1/2)$. 
Then, $\Ascr=\mathsf{SCSR}(\H,\A)$ defined by \cref{alg:sample_comp_real} is a sample compression scheme for $\H$ of size $\size_{\Ascr}(n):=\lceil3.01\log(n+1)/(1/2\gamma-1)^2\gamma\rceil \d_{\RE}(\A,\H,\gamma)$ for $n\in\N$. 
\end{corollary}
\begin{proof}
Let $\bs\in\Z^n$ be an $\H$-realizable sequence.  
By assumption, for $m:=\d_{\RE}(\A,\H,\gamma)$ and any distribution $Q$ over $\bs$, we have 
$$
\E_{T\sim Q^m}[\er_Q(\A(T))]\le\gamma.
$$
By the probabilistic argument, there exists a sequence $\bt\in\bs^m$ such that $\er_Q(\A(\bt,\H))\le \gamma$. Now, we can apply \cref{lem:selection_margin} with $\eps=\sqrt{3/3.01}(1/2\gamma-1)\in(0,1/2\gamma-1)$ to show that $\er_{\bs}\Ascr(\bs,\H)\le \min_{h\in\H}\er_{\bs}(h)$ for all $\bs\in\Z^n$, $n\in\N$. 
\end{proof}
\cref{prop:sample_comp_real} follows directly from the above corollary. 
\begin{proof}[Proof of \cref{prop:sample_comp_real}]
By \eqref{eq:oia}, we have $\d_{\RE}(\A_{\oig},\H,1/3)=O(d)<\infty$. Then, \cref{prop:sample_comp_real} follows directly from \cref{coro:selection_margin}. 
\end{proof}

\section{Label space reduction via multiplicative weights} \label{sec:MW_lemma}

\begin{lemma}
    \label{lem:mw-generic}
    Assume the Multiplicative Weights (see \cref{alg:mul_weights}) algorithm is run for $T$ rounds over the $d$-dimensional simplex $\Delta_d = \{(p_1,\ldots,p_d) \mid \sum_{i=1}^d p_i = 1, p_i \ge 0 \quad \forall i \in [d] \}$ with rewards $r_1,\ldots, r_T \in [0,1]^d$ and step size $\eta \in (0,1]$, and generates a sequence of distributions $p_1,\ldots,p_T \in \Delta_d$. Then for every $j \in [d]$ it holds that
    \begin{align*}
        \sum_{t=1}^T \sum_{i=1}^d p_{t,i} r_t(i) \ge \frac{1}{1+\eta} \sum_{t=1}^T r_t(j) - \frac{1}{\eta(1+\eta)} \log d.
    \end{align*}
\end{lemma}
\begin{proof}
    Define $W_t:=\sum_{i=1}^d w_{t,i}$ for each $t\in[T+1]$. Since $0\le r_t(\cdot),\eta\le 1$ and $p_t \in \Delta_d$, we have
    \begin{align*}
    \log\frac{W_{T+1}}{W_1}=\sum_{t=1}^T\log\frac{W_{t+1}}{W_t}
    =\sum_{t=1}^T\log\sum_{i=1}^d \frac{w_{t,i}e^{\eta r_t(i)}}{W_t}
    =&\sum_{t=1}^T\log\sum_{i=1}^d p_{t,i}e^{\eta r_t(i)}
    \\ \le &
    \sum_{t=1}^T\log\sum_{i=1}^d p_{t,i}\left(1+\eta r_t(i)+\eta^2 r_t(i)^2\right)
    \\ \le &
    \sum_{t=1}^T\eta\left(\sum_{i=1}^d p_{t,i} r_t(h)+\eta\sum_{i=1}^d p_{t,i} r_t(i)^2\right)
    \\ \le &
    \eta(1+\eta)\sum_{t=1}^T\sum_{i=1}^d p_{t,i} r_t(i).
    \end{align*}
    Since $w_{T+1,i}=e^{\eta\sum_{t=1}^Tr_t(i)}$ for each $i \in [d]$, for arbitrary $j \in [d]$, we also have
    \begin{align*}
    \log\frac{W_{T+1}}{W_1}=\log\sum_{i=1}^d e^{\eta\sum_{t=1}^Tr_t(i)}-\log d
    \ge \eta\sum_{t=1}^Tr_t(j)-\log d,
    \end{align*}
    which rearranges to the desired bound after dividing through by $\eta>0$.
\end{proof}

\section{Partial hypothesis class and sample compression} \label{sec:partial_sample_compression}
We first define partial hypotheses for the multiclass setting considered in this paper. 
\begin{definition}[Partial hypothesis \cite{partial_concept}]
Given a label space $\Y$, define $\wt\Y:=\Y\cup\{\bigstar\}$ to be the extended label space where ``$\bigstar$'' represents an undefined value. A function $h:\X\to\wt\Y$ is called a \textbf{partial hypothesis}. If $h(x)=\bigstar$ for some $x\in\X$, then $h$ is ``undefined'' on $x$. The \textbf{support} of $h$ is defined to be $\supp(h)=\{x\in\X:h(x)\in\Y\}$. 
A set $\wt\H\subseteq\wt\Y^\X$ of partial hypotheses is called a partial hypothesis class. 
\end{definition}

Consider a $p$-menu $\mu:\X\to\power^{\Y}$ for some $p\in\N$ and a concept class $\H\subseteq \Y^\X$ of Natarajan dimension $d_N\in\N$. 
For any $h:\X\to\Y$, we can modify it to be a partial concept
\begin{align*}
h_\mu(x):=\begin{cases}
h(x),\textup{ if }h(x)\in\mu(x),\\
\bigstar,\textup{ otherwise},
\end{cases}
\forall x\in\X,
\end{align*}
and consequently, define the partial concept class $\H_\mu:=\{h_\mu:h\in\H\}\subseteq\wt\Y^\X$ based on $\H$. 

% Consider an arbitrary sequence $S\in(\X\times\Y)^n$ of size $n\in\N$ realizable by $\mu$, i.e., $y\in\mu(x)$ for all $(x,y)\in S$. If $\er_S(h)=1$ for all $h\in\H$, we can construct a sample compression scheme of size $0$ on $S$ for $\H$. Thus, we assume that there exists some classifier $h\in\H$ such that $\er_S(h)<1$. Then, consider an arbitrary $\H$-realizable subsequence $R$ of $S$ with size $t\in[n]$. It follows that there exists some $h_R\in \H$ such that $h_R(x)=y\in\mu(x)$ for all $(x,y)\in R$.  

For any sequence $\bt\in(\X\times\Y)^*$, we use $\bt_\X$ to denote the set of distinct $\X$-elements appeared in the sequence $\bt$. 
For any classifier $h\in\wt\Y^\X$ and set $E\subseteq \X$, we can define its restriction to $E$ as
\begin{align*}
h[E]:E\to \wt\Y,\ x\mapsto h(x).
\end{align*}
Then, we can define the total concept class below which certainly could be empty,
\begin{align*}
\H_{\mu,E}:=&
\{h[E]:h\in\H_{\mu}\textup{ such that }E\subseteq\supp(h)\}
\\=&
\{h[E]:h\in\H\textup{ such that }h(x)\in \mu(x)\textup{ for all }x\in E\}
\subseteq\Y^{E}.
\end{align*}
It follows from the assumptions on $\H$ and $\mu$ that $\d_\n(\H_{\mu,E})\le \d_\n(\H)=d_N$ and $|\H_{\mu,E}|_{x}|\le p$ for all $x\in E$ and for any $E\subseteq\X$.

For any sequence $\bt=((x_1,y_1),\dots,(x_{m},y_{m}))\in(\X\times\Y)^m$ of size $m\in\N$ and any pair $(x_{m+1},y_{m+1})\in\X\times\Y$, we define $\bt':=((x_1,y_1),\dots,(x_{m+1},y_{m+1}))$ and
$\bt_\X(x_{m+1}):=\bt'_\X=\bt_\X\cup\{x_{m+1}\}$. 
Then, choosing an arbitrary $y^\dagger\in\Y$ a priori, we can construct an algorithm $\A(\cdot,\H_{\mu}):(\X\times\Y)^*\to\Y^\X$ using $\H_\mu$ by defining for training sequence $\bt=((x_1,y_1),\dots,(x_{m},y_{m}))$ and test point $x_{m+1}$, 
\begin{align*}
\A(\bt,\H_{\mu})(x_{m+1}):=
\begin{cases}
\A_{\oig}(\bt,\H_{\mu,\bt_\X(x_{m+1})})(x_{m+1}),&
\textup{if $\bt$ is $\H_{\mu,\bt_\X(x_{m+1})}$-realizable},\\
y^\dagger,&\textup{otherwise.}
\end{cases}
\end{align*}
Note that $\bt$ is $\H_{\mu,\bt_\X(x_{m+1})}$-realizable if and only if there exists $h\in\H$ such that $h(x_i)=y_i\in\mu(x_i)\ \forall i\in[m]$ and $h(x_{m+1})\in\mu(x_{m+1})$. 
It is not hard to see that algorithmically, $\A$ can be defined without using partial hypotheses. The reason for routing through $\H_\mu$ is to facilitate the theoretical analysis of the selection scheme defined later. 
Equivalently, we can also write $\A$ as the following algorithm $\mathsf{LL}$ which does not involve partial hypotheses. 

\begin{algorithm}[H]
\caption{Learning from list $\mathsf{LL}(\bs,\H,\mu)$} \label{alg:from_list}
\SetKwInput{KwInput}{Input}
\SetKwInput{KwOutput}{Output}
\KwInput{Sequence $\bs=((x_1,y_1),\dots,(x_n,y_n))\in\Z^n$ with $n\in\N$, hypothesis class $\H\subseteq\Y^\X$, menu $\mu:\X\to\power^{\Y}$.}
\KwOutput{A classifier $\hat h\in\Y^\X$.}
Let $\A_{\oig}$ denote the one-inclusion algorithm defined in \cref{alg:oia}\;
Given $x\in\X$, the value of $\hat h(x)\in\Y$ is calculated as follows:
\begin{enumerate}[label=\roman*.]
    \item Let $\H(x,\mu)\gets\{h\in\H:h(x)\in\mu(x)\textup{ and }h(x_i)\in\mu(x_i),\ \forall i\in[n]\}$\;
    \item If $\H(x,\mu)=\emptyset$: let $\hat h(x)\gets y^\dagger$ for arbitrary fixed $y^\dagger\in\Y$ selected a priori;
    
    Else: let $\hat h(x)\gets \A_{\oig}(\bs,\H(x,\mu))(x)$;
\end{enumerate}
\end{algorithm}

Now, we are ready to prove \cref{prop:list_realizable_compression}. 
\begin{proof}[Proof of \cref{prop:list_realizable_compression}]
Let $\bt'=((x_1,y_1),\dots,(x_{n+1}.y_{n+1}))\in\Z^{n+1}$ be $\H_{\mu}$-realizable. Define $\bt=((x_i,y_i))_{i=1}^n$. 
Then, $\bt'$ is also $\H_{\mu,\bt_{\X}(x_{n+1})}$-realizable (note that $\bt_{\X}(x_{n+1})=\bt'_{\X}$) and thus, $\A(\bt'_{-i},\H_{\mu})(x_{i})=\A_{\oig}(\bt'_{-i},\H_{\mu,\bt'_\X})(x_{i})$ for all $i\in[n+1]$, where $\bt'_{-i}:=((x_j,y_j))_{j\in[n+1]\backslash\{i\}}$. 
By \cite[Lemma 17]{brukhim2022characterization}, since $\mu$ has size $p$ and $\d_\n(\H_{\mu,\bt'_\X})\le d_N$, we have
\begin{align*}
\sum_{i=1}^{n+1}\ind_{\A(\bt'_{-i},\H_{\mu})(x_{i})\neq y_i}=
\sum_{i=1}^{n+1}\ind_{\A_{\oig}(\bt'_{-i},\H_{\mu,\bt'_\X})(x_{i})\neq y_i}\le 20d_N\log(p).
\end{align*}
Any distribution $P$ over $\bt'$ is $\H_{\mu}$-realizable. Thus, for any $m\in\N$, sampling 
$$
T'=((X_1,Y_1),\dots,(X_{m+1},Y_{m+1}))\sim P^{m+1}
$$ 
and defining $T:=((X_1,Y_1),\dots,(X_{m},Y_{m}))$, we have that $T'$ is $\H_{\mu}$-realizable almost surely. 
It follows from the above inequality that
\begin{align*}
\E\left[\er_{P}(\A(T,\H_{\mu}))\right]
=&\P\left(\A(T;\H_{\mu})(X_{m+1})\neq Y_{m+1}\right)
\\=&
\E\left[\frac{1}{m+1}\sum_{i=1}^{m+1}\ind_{\A(T'_{-i},\H_{\mu})(X_i)\neq Y_i}\right]
\\ \le &
\frac{20d_N\log(p)}{m+1}.
\end{align*}
Setting $m=\lfloor60d_N\log(p)\rfloor$, for any distribution $P$ over $\bt'$, we have $
\E\left[\er_{P}(\A(T,\H_{\mu}))\right]\le\frac{1}{3}$, indicating that there exists some deterministic sequence $\br\in (\bt')^m$ such that $\er_P(\A(\br;\H_{\mu}))\le \frac{1}{3}$. 
Then, by \cref{lem:selection_margin}, there exist an integer $\ell=\Theta(\log n)$ and a selection scheme $(\kappa,\rho)$ which for $(\bt^{(1)},\dots,\bt^{(\ell)})\in (\bt')^{m\ell}$ selected by $\kappa$ from $\bt'$ with $\bt^{(i)}\in (\bt')^{m}$ for each $i\in[\ell]$, 
$$
\hat h=\rho((\bt^{(1)},\dots,\bt^{(\ell)}),\H,\mu):=\Maj(\A(\bt^{(1)},\H_\mu),\dots,\A(\bt^{(\ell)},\H_\mu))
$$
satisfies that $\hat h(x)=y$ for all $(x,y)\in \bt'$. 
As described above, the size of this selection scheme is $k(n)=m\ell=\Theta(d_N\log(p)\log(n))$. Furthermore, following \cref{alg:sample_comp_real} and \cref{alg:from_list}, we can define this selection scheme explicitly in \cref{alg:sample_comp_list}. It is clear that $\mathsf{LSCS}(\mu,\H)$ is an $l^\mu$-sample compression scheme for $\H$. 
\end{proof}

\begin{algorithm}[H]
\caption{$l^\mu$-sample compression scheme via the one-inclusion algorithm $\mathsf{LSCS}(\mu,\H)$} \label{alg:sample_comp_list}
\SetKwInput{KwInput}{Input}
\SetKwInput{KwOutput}{Output}
\KwInput{Menu $\mu:\X\to\power^{\Y}$ of size $p\in\N$, hypothesis class $\H\subseteq\Y^\X$ of Natarajan dimension $d_N\in\N$.}
\KwOutput{An $l^\mu$-sample compression scheme $(\kappa,\rho)$ for $\H$ of size $k(n)$.}
Let $\mathsf{LL}$ denote the algorithm defined in \cref{alg:from_list}\;
Let $m\gets \lfloor60d_N\log(p)\rfloor$ and define $k(n):=\lceil40\log(n+1)\rceil m$ for all $n\in\N$\;
Given $\bt=(z_1,\dots,z_n)\in\Z^n$ with $n\in\N$, $\rho(\bt)\in\Y^\X$ is calculated as follows:
\begin{enumerate}[label=\roman*.]
    \item Let $a\gets \lfloor n/m\rfloor$;
    \item If $a<1$: let $\rho(\bt)\gets h^\dagger$ for arbitrary fixed $h^\dagger\in\Y^\X$ selected a priori;

    Else: let $\bt^{(i)}\gets (z_{(i-1)m+1},\dots,z_{im})$ for each $i\in[a]$ and 
    $$\rho(\bt)\gets \Maj(\mathsf{LL}(\bt^{(1)},\H,\mu),\dots,\mathsf{LL}(\bt^{(a)},\H,\mu));
    $$
\end{enumerate}
Given $\bs\in\Z^n$ with $n\in\N$, $\kappa(\bs)$ is calculated as follows:
\begin{enumerate}[label=\roman*.]
    \item Let $\bs':=((x,y)\in\bs:y\in\mu(x))$;
    \item Let $\tilde h\in\argmax_{h\in\H}|\bs'(h)|$ and $\tilde \bs\gets \bs'(\tilde h)$;
    \item Find $\bt\in\tilde\bs^{k(n)}$ such that $\rho(\bt)$ realizes $\tilde\bs$ and let $\kappa(\bs)\gets \bt$. 
\end{enumerate}
\end{algorithm}

\section{Lower bound} \label{sec:lb}
\begin{proof}[Proof of the lower bound in \cref{thm:main}]
By \cite[Theorem 3.5]{daniely:15}, we have
\begin{align} \label{eq:lb1}
\M_{\H}^{\AG}(\eps,\delta)\ge \Omega((d_N+\log(1/\delta))/\eps^2).
\end{align}
Now we can assume that $d\ge18e$. 
By \cite[Theorem 2]{daniely2014optimal} and the relationship between transductive error rate and expected error rate, we have
\begin{align*}
\mu_\H(d)/d\ge \eps^{\RE}_\H(d-1)> 1/9e.
\end{align*}
Then, by the definition of $\mu_\H$ in \cref{def:deg}, there exists $\bx=(x_1,\dots,x_d)\in\X^d$ and $V\subseteq\H|_{\bx}$ such that $k:=|V|<\infty$ and $\avgdeg(\G(V))\ge d/9e$. 
Let $E$ denote the edge set of $G:=\G(V)$. Then, for any $i\in[d]$, we have
\begin{align*}
\frac{1}{k}\sum_{e\in\E:|e|\ge2,\dir(e)\neq i}|e|\ge\frac{1}{k}\sum_{v\in V}(\deg(v;G)-1)=\avgdeg(G)-1\ge d/9e-1\ge d/18e. 
\end{align*}
For any $v\in V$ and $i\in[d]$, define $b(v,i):=\ind\{\exists e\in E\textup{ s.t. }|e|\ge2,\ v\in e,\ \dir(e)=i\}$. 
Enumerate the vertices in $V$ as $v_1,\dots,v_k$ with $v_j=(y_{j,1},\dots,y_{j,d})\in\Y^d$. 
Sample $J\sim\unif([K])$. Then, we have 
\begin{align*}
\sum_{i=2}^d\E[b(v_J,i)]
=\frac{1}{k}\sum_{i=2}^d\sum_{j=1}^kb(v_j,i)
=\frac{1}{k}\sum_{e\in E:|e|\ge2,\dir(e)\neq 1}|e|\ge d/18e.
\end{align*}
For any $i\ge [d]$, $i_1,\dots,i_m\in[d]\backslash\{i\}$ with $m\in\N$, and $y\in\Y$, the definition of $b$ ensures that
\begin{align*}
\P(y_{J,i}\neq y\mid y_{J,i_1},\dots,y_{J,i_m},b(v_J,i)=1)\ge 1/2.
\end{align*}

For any $\epsilon\in(0,1/144e)$, consider the categorical distribution $Q_{\epsilon}$ over $[d]$ defined by $Q_{\epsilon}(\{1\})=1-144e\epsilon$ and $Q_{\epsilon}(\{i\})=\frac{144e\epsilon}{d-1}$ for $i\in[d]\backslash\{1\}$. 
Sample $(I_1,\dots,I_n,I)\sim(Q_\eps)^{n+1}$ with $n\in\N$ independently. Let $S:=((x_{I_1},y_{J,I_1}),\dots,(x_{I_n},y_{J,I_n}))$. 
For any multiclass learner $\A$ and $i\in[d]$, by the above inequality, we have
\begin{align*}
&\P(\A(S,\H)(x_i)\neq y_{J,i}\mid b(v_J,i)=1,I_{t}\neq i, \forall t\in[n])
\\ = &
\E[\P(\A(S,\H)(x_i)\neq y_{J,i}\mid \A, S, b(v_J,i)=1,I_{t}\neq i, \forall t\in[n])
\mid b(v_J,i)=1,I_{t}\neq i, \forall t\in[n]]
\\ \ge &
1/2.
\end{align*}
It follows that
\begin{align*}
&\P(\A(S,\H)(x_I)\neq y_{J,I}, I\neq 1)
\\ \ge &
\P(\A(S,\H)(x_I)\neq y_{J,I},\ I\neq 1,\ I_{t}\neq I, \forall t\in[n])
\\=&
\sum_{i=2}^d\P(\A(S,\H)(x_i)\neq y_{J,i},\ I=i,\ I_{t}\neq i, \forall t\in[n])
\\=&
\sum_{i=2}^dQ_\eps(\{i\})(1-Q_\eps(\{i\}))^n\P(\A(S,\H)(x_i)\neq y_{J,i}\mid I_{t}\neq i, \forall t\in[n])
\\ \ge &
Q_\eps(\{2\})(1-Q_\eps(\{2\}))^n\sum_{i=2}^d\P(b(v_J,i)=1)\cdot
\\&\quad\quad\quad\quad\quad\quad\quad\quad\quad\quad\quad\quad
\P(\A(S,\H)(x_i)\neq y_{J,i}\mid b(v_J,i)=1, I_{t}\neq i, \forall t\in[n])
\\ \ge &
\frac{Q_\eps(\{2\})(1-Q_\eps(\{2\}))^n}{2}\sum_{i=2}^d\E[b(v_J,i)]
\\ > &
4\eps\left(1-\frac{144e\eps}{d-1}\right)^n.
\end{align*}
Since $0<\frac{144e\eps}{d-1}<\frac{1}{3}$, we have $\log\left(1-\frac{144e\eps}{d-1}\right)>-\frac{288e\eps}{d-1}$ and $\frac{(d-1)\log(2)}{288e\epsilon}<\frac{\log(1/2)}{\log(1-144e\epsilon/(d-1))}$. 
Then, for any $n\le \frac{(d-1)\log(2)}{288e\epsilon}$, we have $\left(1-\frac{144e\epsilon}{d-1}\right)^n>
\left(1-\frac{144e\epsilon}{d-1}\right)^{\frac{\log(1/2)}{\log(1-144e\epsilon/(d-1))}}=\frac{1}{2}$ and
\begin{align*}
\P(\A(S,\H)(x_I)\neq y_{J,I}, I\neq 1)
> 2\eps. 
\end{align*}
By the probabilistic argument, there exists some $j\in[k]$ such that for $S_j:=((x_{I_t},y_{j,I_t})_{t=1}^n)$,
\begin{align*}
\P(\A(S_j,\H)(x_I)\neq y_{j,I}, I\neq 1)> 2\eps.
\end{align*}
Let $P$ denote the distribution over $\X\times\Y$ such that 
\begin{align*}
P(\{(x_1,y_{j,1})\})=Q_\eps(\{1\})=1-8\eps,\ 
P(\{(x_i,y_{j,i})\})=Q_\eps(\{i\})=8\eps/(d-1),\ \forall i\in[d]\backslash\{1\}.
\end{align*}
Then, we know that $(S_j,(x_I,y_{j,I}))\sim P^{n+1}$. 

Define a new algorithm $\A'$ by letting $\A'(\bs,\H)(x_1):=y_{j,1}$ and $\A'(\bs,\H)(x):=\A(\bs,\H)(x)$ for any $\bs\in\Z^*$ and $x\in\X\backslash\{x_1\}$. 
Then, we have
\begin{align}
\notag
\E[\er_P(\A'(S_j,\H))]=&\P(\A'(S_j,\H)(x_I)\neq y_{j,I},I\neq 1)
\\=&  \label{eq:eer_lb}
\P(\A(S_j,\H)(x_I)\neq y_{j,I},I\neq 1)
>2\eps
\end{align}
and
\begin{align*}
\er_P(\A'(S_j,\H))\le \P(I\neq 1)= 144e\eps.
\end{align*}
Suppose that 
\begin{align*}
\P(\er_P(\A(S_j,\H))>\eps)\le 1/144e.
\end{align*}
Since $\er_P(\A'(S_j,\H))\le \er_P(\A(S_j,\H))$, we also have $\P(\er_P(\A'(S_j,\H))>\eps)\le 1/144e$. 
It then follows that
\begin{align*}
\E[\er_P(\A'(S_j,\H))]\le
144e\eps\P(\er_P(\A'(S_j,\H))>\eps)
+\eps\le 2\eps,
\end{align*}
which contradicts \eqref{eq:eer_lb}. 
Thus, we must have 
\begin{align*}
\P(\er_P(\A(S_j,\H))>\eps)> 1/144e
\end{align*}
for $S_j\sim P^n$ if $n\le \frac{(d-1)\log_e(2)}{288e\epsilon}$. 
This proves that for $\eps\in(1/144e)$ and $\delta\le 1/144e$, 
\begin{align} \label{eq:lb2}
\M_{\H}^{\AG}(\eps,\delta)>\frac{(d-1)\log_e(2)}{288e\epsilon}=\Omega(d/\eps). 
\end{align}
By \eqref{eq:lb1} and \eqref{eq:lb2}, we have
\begin{align*}
\M_{\H}^{\AG}(\eps,\delta)\ge \Omega\left(\frac{d_N+\log(1/\delta)}{\eps^2}+\frac{d}{\eps}\right).
\end{align*}

\end{proof}

%\section{Simplification of bounds} \label{sec:simp}
%\input{simplification}

\section{Technical lemmas} \label{sec:lemmas}
\begin{lemma} \label{lem:cond_chernoff}
Let $(\F_i)_{i=0}^n$ ($n\in\N$) be a filtration and $(X_i)_{i\in[n]}\in[0,1]^n$ be a sequence of $(\F_i)_{i\in[n]}$-adapted random variables. Then for any $a,\delta\in(0,1]$ and $b>0$, we have 
\begin{align*}
\P\left(\max_{j\in[n]}\left(\sum_{i=1}^jX_i-(1+a)\sum_{i=1}^j\E[X_i\mid\F_{i-1}]\right)\le \frac{\log(1/\delta)}{a}\right)\ge 1-\delta
\end{align*}
and
\begin{align*}
\P\left(\min_{j\in[n]}\left(\sum_{i=1}^jX_i-(1-b)\sum_{i=1}^j\E[X_i\mid\F_{i-1}]\right)\ge -\frac{\log(1/\delta)}{2b}\right)\ge 1-\delta
\end{align*}
\end{lemma}
\begin{proof}
Define $Z_j:=\E[X_j\mid\F_{j-1}]$ and $M_j:=\exp(a(\tsum_{i=1}^jX_i-(1+a)\tsum_{i=1}^jZ_i))$ for each $j\in[n]$. Define $M_0:=1$. 
Since $0\le a,X_j\le 1$ for each $j\in[n]$, we have $\exp(aX_j)\le 1+aX_j+a^2X_j^2\le 1+(a+a^2)X_j$ and
\begin{align*}
\E[\exp(aX_j)\mid\F_{j-1}]\le 1+(a+a^2)\E[X_j\mid\F_{j-1}]=1+(a+a^2)Z_j\le \exp(a(1+a)Z_j).
\end{align*}
Since $(X_j)_{j\in[n]}$ is $(\F_j)_{j\in[n]}$-adapted, we know that $(M_j)_{j=0}^n$ is $(\F_j)_{j=0}^n$-adapted and for any $j\in[n]$, 
\begin{align*}
\E[M_j\mid\F_{j-1}]=&
\E[\exp(a(\tsum_{i=1}^jX_i-(1+a)\tsum_{i=1}^jZ_i))\mid\F_{j-1}]
\\ = &
\exp(a(\tsum_{i=1}^{j-1}X_i-(1+a)\tsum_{i=1}^{j}Z_i))\E[\exp(aX_j)\mid\F_{j-1}]
\\ \le &
\exp(a(\tsum_{i=1}^{j-1}X_i-(1+a)\tsum_{i=1}^{j-1}Z_i))
=M_{j-1}.
\end{align*}
Therefore, $(M_j)$ is an $(\F_j)$-supermartingale. By Doob's maximal  inequality for supermartingales (see, e.g., \cite[Theorem 3.9]{lattimore2020bandit}), we have
\begin{align*}
\delta=\delta\E[M_0]\ge&\P\left(\max_{j\in[n]} M_j\ge 1/\delta\right)
=\P\left(\max_{j\in[n]}\left(\sum_{i=1}^jX_i-(1+a)\sum_{i=1}^jZ_i\right)\ge \frac{\log(1/\delta)}{a}\right).
\end{align*}

Following the same idea, we define $N_j:=\exp(-2b(\tsum_{i=1}^jX_i-(1-b)\tsum_{i=1}^jZ_i))$ for $j\in[n]$ and $N_0:=0$. 
Since $b\ge0$ and $X_j\in[0,1]$ for each $j\in[n]$, we have $\exp(-2bX_j)\le 1-2bX_j+2b^2X_j^2\le 1-2b(1-b)X_j$ and
\begin{align*}
\E[\exp(-2bX_j)\mid\F_{j-1}]\le 1-2b(1-b)\E[X_j\mid\F_{j-1}]=1-2b(1-b)Z_j\le \exp(-2b(1-b)Z_j).
\end{align*}
Since $(X_j)_{j\in[n]}$ is $(\F_j)_{j\in[n]}$-adapted, we know that $(N_j)_{j=0}^n$ is $(\F_j)_{j=0}^n$-adapted and for any $j\in[n]$, 
\begin{align*}
\E[N_j\mid\F_{j-1}]=&
\E[\exp(-2b(\tsum_{i=1}^jX_i-(1-b)\tsum_{i=1}^jZ_i))\mid\F_{j-1}]
\\ = &
\exp(-2b(\tsum_{i=1}^{j-1}X_i-(1-b)\tsum_{i=1}^{j}Z_i))\E[\exp(-2bX_j)\mid\F_{j-1}]
\\ \le &
\exp(-2b(\tsum_{i=1}^{j-1}X_i-(1-b)\tsum_{i=1}^{j-1}Z_i))
=N_{j-1}.
\end{align*}
Therefore, $(N_j)$ is an $(\F_j)$-supermartingale. By Doob's maximal  inequality for supermartingales, we have
\begin{align*}
\delta=\delta\E[N_0]\ge&\P\left(\max_{j\in[n]} N_j\ge 1/\delta\right)
=\P\left(\min_{j\in[n]}\left(\sum_{i=1}^jX_i-(1-b)\sum_{i=1}^jZ_i\right)\le -\frac{\log(1/\delta)}{2b}\right).
\end{align*}
\end{proof}

\begin{lemma} \label{lem:simp}
If $n_{\eps,\delta}$ satisfies \eqref{eq:ub_raw}, then it also satisfies 
\begin{align*}
n_{\eps,\delta}=O\bigg(&\frac{d_N\log^3(d_N/\eps)+\log(1/\delta)}{\eps^2}
+\frac{d\log^2(d/\eps)}{\eps}\bigg).
\end{align*}
If $d\le C(d_{DS})^{1+\alpha}\log^{\beta}(d_{DS})$ for constants $\alpha>0$, $\beta\ge0$, and $C\ge1$, \eqref{eq:ub_raw} also implies that
\begin{align*}
n_{\eps,\delta}=O\bigg(
\frac{d_N\log^3(1/\eps)+\log(1/\delta)}{\eps^2}
+\frac{(d_{DS})^{1+\alpha}\log^\beta(d_{DS})\log^2(d_{DS}/\eps)}{\eps}
\bigg).
\end{align*}
\end{lemma}
\begin{proof}
For notational convenience, define $a:=d/\eps$, $b:=\log(1/\delta)/\eps$,  $c:=d_N/\eps$, and 
$$
A:=c\log(a+b)\log^2(c\log(a+b))/\eps+a\log^2(a+b)+b/\eps.
$$
Then, \eqref{eq:ub_raw} can be written as $n_{\eps,\delta}=O(A)$. 
Without loss of generality, we can assume that $a,b,c>e$. 
If $a\log(a+b)\le c\log^2(c\log(a+b))/\eps$, we have $a\le (c/\eps)\frac{\log^2(c\log(a+b))}{\log(a+b)}$ and
\begin{align*}
\log(a)\le &\log(c/\eps)+2\log\log(c\log(a+b))-\log\log(a+b)
\\ \le & \log(c/\eps)+2\log\log(c)+2\log\log\log(a+b)-\log\log(a+b)
\\ \le &
2\log(d_N/\eps)+2\log\log(c)
\le 3\log(c). 
\end{align*}
It follows that $a+b\le c^3+b\le (c+b)^3$, $\log(a+b)\le 3\log(c+b)$, and
\begin{align*}
A\le 3c\log(c+b)\log^2(3c\log(c+b))/\eps+a\log^2(a+b)+b/\eps
\end{align*}
if $a\log(a+b)\le c\log^2(c\log(a+b))/\eps$. Therefore, we can conclude that
\begin{align} \label{eq:Aub1}
A\le &3c\log(c+b)\log^2(3c\log(c+b))/\eps+2a\log^2(a+b)+b/\eps.
\end{align}

If $\log(b)\ge c$, then $c\le b$, $\log\log(c+b)\le \log\log(2b)$, $3\log(b)\le 2b$, 
\begin{align*}
&\log^2(3c\log(c+b))\le (\log(3\log(b))+\log\log(2b))^2\le 4\log(3\log(b)), \textup{ and}\\
&c\log(c+b)\log^2(3c\log(c+b))\le 4\log(b)\log(2b)\log(3\log(b))\le 4\log^3(2b)\le 11b.
\end{align*}
If $\log(b)<c$, then $\log\log(c+b)\le \log(\log c+\log b)\le \log(2c)$, 
\begin{align*}
&\log^2(3c\log(c+b))\le (\log(3c)+\log(2c))^2\le 4\log^2(3c), \textup{ and}\\
&c\log(c+b)\log^2(3c\log(c+b))
\le 4c\log(c+b)\log^2(3c).
\end{align*}
If $4c\log(c+b)\log^2(3c)\ge b$, then we have
\begin{align*}
\log(4c)+\log\log(c+b)+2\log\log(3c)\ge\log(b).
\end{align*}
Since $\log\log(c+b)\le \log\log(c)+\log\log(b)\le \log(4c)/3+\log(b)/e$ and $\log\log(3c)\le \log(4c)/e$, the above inequalities also imply that
\begin{align*}
&(4/3+2/e)\log(4c)\ge(1-1/e)\log(b),\\
&\log(b)\le 3.3\log(4c),\\
&\log(c+b)\le \log(c+(4c)^{3.3})\le 4\log(3c),\textup{ and}\\
&4c\log(c+b)\log^2(3c)\le 16c\log^3(3c).
\end{align*}
Summarizing the above analysis, we have
\begin{align} \label{eq:Aub2}
c\log(c+b)\log^2(3c\log(c+b))\le \max\{11b,16c\log^3(3c)\}. 
\end{align}

If $2b\le a\eps\log^2(a+b)$, then $a\log^2(a+b)\ge 2b/\eps$, 
\begin{align*}
&\log(2b)\le \log(a\eps)+2\log\log(a+b)\le
(1+2/e)\log(a)+0.6\log(2b),\\
&\log(2b)\le 2.5(1+2/e)\log(a),\textup{ and}\\
&\log^2(a+b)\le (\log a+\log b)^2\le 29\log^2(a).
\end{align*}
It follows that
\begin{align} \label{eq:Aub3}
a\log^2(a+b)\le \max\{2b/\eps, 29a\log^2(a)\}.
\end{align}
By \eqref{eq:Aub1}, \eqref{eq:Aub2}, and \eqref{eq:Aub3}, we have
\begin{align*}
A=O((c\log^3(c)+b)/\eps+a\log^2(a))=O\bigg(&\frac{d_N\log^3(d_N/\eps)+\log(1/\delta)}{\eps^2}
+\frac{d\log^2(d/\eps)}{\eps}\bigg).
\end{align*}

Now, suppose that $d=C(d_{DS})^{1+\alpha}\log^{\beta}(d_{DS})$ for constants $\alpha>0$, $\beta\ge0$, and $C\ge 1$. 
% Since $d=\Omega(d_{DS})$ and $d_{DS}\ge d_N$, there exists some constant $C'>0$ such that $d>C'd_N$. 
If $a\log^2(a)\le c\log^3(c)/\eps$, we have $a/c\le (1/\eps)\frac{\log^3(c)}{\log^2(a)}$ and
\begin{align*}
\log(d/d_{N})=\log(a/c)\le &\log(1/\eps)+3\log\log(c)-2\log\log(a)
\\ \le &
\log(1/\eps)+3\log\log(d_N/\eps)
\le (1+3/e)\log(1/\eps)+3\log\log(d_N).
\end{align*}
Since $d=C(d_{DS})^{1+\alpha}\log^{\beta}(d_{DS})$, we have 
\begin{align*}
\alpha\log(d_{DS})\le \log(d/d_N)\le 2\log(1/\eps)+2\log\log(d_N)
\le (1+3/e)\log(1/\eps)+3\log\log(d_{DS}).
\end{align*}
By \cref{lem:logeineq}, if $d_{DS}\ge (4e/\alpha)^{8/\alpha}$, we have $\alpha\log(d_{DS})\ge4\log\log(d_{DS})$ and by the above inequality, 
\begin{align*}
\alpha \log(d_{DS})\le (1+3/e)\log(1/\eps)+(3\alpha/4)\log(d_{DS}).
\end{align*}
It follows that 
\begin{align*}
&\log(d_N)\le \log(d_{DS})\le (4(1+3/e)/\alpha)\log(1/\eps)\textup{ and}\\
&\log(d_N/\eps)\le (4(1+3/e)/\alpha+1)\log(1/\eps).
\end{align*}
In summary, we have 
\begin{align*}
c\log^3(c)/\eps\le \max\{a\log^2(a), (4(1+3/e)/\alpha+1)^3c\log^3(1/\eps)/\eps\}.
\end{align*}
Since $\alpha,\beta,C$ are constants, we immediately have
\begin{align*}
A= & O\bigg(\frac{d_N\log^3(1/\eps)+\log(1/\delta)}{\eps^2}
+\frac{d\log^2(d/\eps)}{\eps}\bigg)
\\=&
O\bigg(\frac{d_N\log^3(1/\eps)+\log(1/\delta)}{\eps^2}
+\frac{(d_{DS})^{1+\alpha}\log^\beta(d_{DS})\log^2(d_{DS}/\eps)}{\eps}\bigg).
\end{align*}

\end{proof}

\begin{lemma} \label{lem:logeineq}
If $x>0$ satisfies $x\le a\log(x/a)+b$ for some $a,b>0$, then, we have $x<2a+2b$. It follows that $x\ge2a\log(ea)+2b$ implies that $x>a\log(x)+b$. 
\end{lemma}
\begin{proof}
Define $f(x):=x-a\log(x/a)-b$ for $x>0$. Then, we have $f'(x)=\frac{x-a}{x}$, which implies that $f$ decreases with $x$ for $x\in(0,a)$ and increases with $x$ for $x>a$. Since $2a+2b>a$, it suffices to prove that $f(2a+2b)> 0$. Indeed, 
\begin{align*}
f(2a+2b)=(2-\log2)a+b-a\log((a+b)/a)>a\left((a+b)/a-\log((a+b)/a)\right)\ge 0. 
\end{align*}
\end{proof}

\begin{lemma} \label{lem:log_ineq}
For any $a,b>0$, if $x\ge 200a\Log^2(a)+23b$, then we have $x\ge a\log^2(x)+b$, where $\Log:\R\to\R,\ t\mapsto\log(\max\{t,e\})$.
\end{lemma}
\begin{proof}
Define the function 
\begin{align*}
f(x):=x-a\log^2(x)+b,\ \forall x>0.
\end{align*}
Then, we have
\begin{align*}
f'(x)=\frac{x-2a\log x}{x},\ \forall x>0.
\end{align*}
Define the function
\begin{align*}
g(x):=x-2a\log x,\ \forall x>0.
\end{align*}
Then, we have
\begin{align*}
g'(x)=\frac{x-2a}{x},\ \forall x>0,
\end{align*}
which implies that $g(x)$ is an increasing function of $x$ for $x\ge 2a$. If follows that for any $x\ge 200a\Log(a)>2a$,
\begin{align*}
g(x)\ge g(200a\Log(a))=200a\Log(a)-2a(\log(a)+\log(\Log(a))+\log(200)). 
\end{align*}
If $a<e$, we have
\begin{align*}
g(200a\Log(a))=2(100-\log(a)-\log200)a>2(99-\log200)a>0.
\end{align*}
If $a\ge e$, we have
\begin{align*}
g(200a\Log(a))=2(99\log(a)-\log(\log(a))-\log200)a
\ge 2a(98-\log200)>0.
\end{align*}
Therefore, we have $f'(x)=g(x)/x>0$ for all $a>0$ and $x\ge 200a\Log(a)$, which implies that 
\begin{align*}
f(x)\ge f(200a\Log^2(a)+23b)
\end{align*}
for any $x\ge c:=200a\Log^2(a)+23b$. It follows that we only need to show that $c\ge a\log^2(c)+b$. 

Suppose that $a\le e$. Then, we have $c=200a+23b$,
\begin{align*}
\log(c)=\log(200a+23b)\le \log(2\max\{200a,23b\})\le \log2+\max\{\log(23b),1+\log200\},
\end{align*}
\begin{align*}
\log^2(c)\le 2\log^22+2\max\{\log^2(23b),(1+\log200)^2\}
< 81+4\log^2(23)+4\log^2(b),
\end{align*}
and
\begin{align*}
a\log^2(c)+b< 
121a+4a\log^2(b)+b.
\end{align*}
If $b\le e$, we have $a\log^2(c)+b<125a+b\le 200a+23b=c$. 
If $b> e$, we have $a\log^2(c)+b<121a+(4e+1)b< 200a+12b<c$. Thus, we have $a\log^2(c)+b<c$ for $a\le e$. 

Suppose that $a>e$. Then, we have $c=200a\log^2(a)+23b$, 
\begin{align*}
\log(c)=\log(200a\log^2(a)+23b)\le \log2+\max\{\log a+2\log\log(a)+\log200,\log b+\log23\}, 
\end{align*}
\begin{align*}
a\log^2(c)+b\le 
a+2a\max\{3\log^2(a)+12(\log\log(a))^2+3\log^2(200),2\log^2(b)+2\log^2(23)\}+b
\end{align*}
If $200a\log^2(a)\ge 23b$, we have
\begin{align*}
a\log^2(c)+b\le &
(1+6\log^2(200))a+6a(\log^2(a)+(2\log\log(a))^2)+b
\\ \le &
(1+6\log^2(200))a+12a\log^2(a)+b
\\ \le &
(13+6\log^2(200))a\log^2(a)+b
<c.
\end{align*}
On the other hand, if $23b> 200a\log^2(a)>200e$, we have $b>200e/23>23$, $b>0.746\log^3(b)$, and
\begin{align*}
a\log^2(c)+b\le
(4\log^2(23)+1)a+b+4a\log^2(b).
\end{align*}
When $a\le 5b/\log^2(b)$, we have
\begin{align*}
a\log^2(c)+b\le 
(4\log^2(23)+1)a+b+4a\log^2(b)
< 41a+21b<c.
\end{align*}
When $a>5b/\log^2(b)$, i.e., $b<\frac{a}{5}\log^2(b)$, we have $a>3.73\log(b)$, 
\begin{align*}
\log(b)<\log(a)+2\log(\log(b))-\log(5)
<3\log(a)-4.24,
\end{align*}
and
\begin{align*}
a\log^2(c)+b\le (4\log^2(23)+1)a+b+4a\log^2(b)< 
(41+36\log^2(a))a+b
<77a\log^2(a)+b<c.
\end{align*}

In conclusion, for all choices of $a,b>0$, we have $c\ge a\log^2(c)+b$, which implies that $x\ge a\log^2(x)+b$ for any $x\ge 200a\Log^2(a)+23b$. 
\end{proof}

\bibliographystyle{alpha}
\bibliography{ref}

@inproceedings{Vapnik68,
title = {On the uniform convergence of relative frequencies of events to their probabilities},
	author = {V. Vapnik and A. Chervonenkis},
	booktitle = "Proc. USSR Acad. Sci.",
	year = 1968
}

@article{natarajan:89,
author = {B. K. Natarajan},
title = {On learning sets and functions},
journal = {Machine Learning},
volume = 4,
pages = "67--97",
year = 1989
}

@article{ben-david:95,
author = {S. Ben-David and N. Cesa-Bianchi and D. Haussler and P. Long},
title = {Characterizations of learnability for classes of $\{0,\ldots,n\}$-valued functions},
journal = {Journal of Computer and System Sciences}, 
volume = 50,
pages = "74--86",
year = 1995
}

@article{daniely:15,
author = {A. Daniely and S. Sabato and S. Ben-David and S. Shalev-Shwartz},
title = {Multiclass Learnability and the {ERM} Principle},
journal = {Journal of Machine Learning Research},
volume = 16,
number = 12,
pages = "2377--2404",
year = 2015
}

@inproceedings{DanielySS12,
	author = {Amit Daniely and Sivan Sabato and Shai Shalev{-}Shwartz},
	bibsource = {dblp computer science bibliography, https://dblp.org},
	booktitle = {NIPS},
	pages = {494--502},
	title = {Multiclass Learning Approaches: {A} Theoretical Comparison with Implications},
	year = {2012}}

@inproceedings{daniely2015inapproximability,
	author = {Daniely, Amit and Schapira, Michael and Shahaf, Gal},
	booktitle = {STOC},
	pages = {401--408},
	title = {Inapproximability of truthful mechanisms via generalizations of the VC dimension},
	year = {2015}}

@inproceedings{long01,
  title={On agnostic learning with $\{$0,*, 1$\}$-valued and real-valued hypotheses},
  author={Long, Philip M},
  booktitle={Computational Learning Theory: 14th Annual Conference on Computational Learning Theory, COLT 2001 and 5th European Conference on Computational Learning Theory, EuroCOLT 2001 Amsterdam, The Netherlands, July 16--19, 2001 Proceedings 14},
  pages={289--302},
  year={2001},
  organization={Springer}
}

@article{HopkinsKLM24,
  author       = {Max Hopkins and
                  Daniel M. Kane and
                  Shachar Lovett and
                  Gaurav Mahajan},
  title        = {Realizable Learning is All You Need},
  journal      = {TheoretiCS},
  volume       = {3},
  year         = {2024},
  url          = {https://doi.org/10.46298/theoretics.24.2},
  doi          = {10.46298/THEORETICS.24.2},
  bibsource    = {dblp computer science bibliography, https://dblp.org}
}

@article{haussler:95,
	author = {David Haussler and Philip M. Long},
	journal = {J. Comb. Theory, Ser. {A}},
	number = {2},
	pages = {219--240},
	title = {A Generalization of {S}auer's Lemma},
	volume = {71},
	year = {1995}}

@Article{haussler:94,
author = {D. Haussler and N. Littlestone and M. Warmuth},
title = {Predicting $\{0,1\}$-Functions on Randomly Drawn Points},
journal = {Information and Computation},
year = {1994},
volume = {115},
number = 2,
pages = "248--292"
}

@article{freund:97b,
author = {Y. Freund and R. E. Schapire},
title = {A Decision-theoretic Generalization of On-line Learning and an Application to Boosting},
journal = {Journal of Computer and System Sciences},
year = 1997,
volume = 55,
number = 1,
pages = "119--139"
}

@article{floyd:95,
author = {S. Floyd and M. Warmuth},
title = {Sample Compression, Learnability, and the {V}apnik-{C}hervonenkis Dimension},
journal = {Machine Learning},
volume = 21,
number = 3,
pages = "269--304",
year = 1995
}

@article{hanneke:15b,
author = {S. Hanneke and L. Yang},
title = {Minimax Analysis of Active Learning},
journal = {Journal of Machine Learning Research}, 
volume = 16,
number = 12, 
pages = "3487--3602",
year = 2015
}

@article{littlestone:94,
author = {N. Littlestone and M. K. Warmuth},
title = {The Weighted Majority Algorithm},
journal = {Information and Computation},
volume = 108,
number = 2,
year = 1994,
pages = "212--261"
}

@book{cesa-bianchi:06,
author = {N. Cesa-Bianchi and G. Lugosi},
title = {Prediction, Learning, and Games},
publisher = {Cambridge University Press},
year = 2006
}

@inproceedings{brukhim2022characterization,
  title={A characterization of multiclass learnability},
  author={Brukhim, Nataly and Carmon, Daniel and Dinur, Irit and Moran, Shay and Yehudayoff, Amir},
  booktitle={2022 IEEE 63rd Annual Symposium on Foundations of Computer Science (FOCS)},
  pages={943--955},
  year={2022},
  organization={IEEE}
}

@inproceedings{daniely2014optimal,
  title={Optimal learners for multiclass problems},
  author={Daniely, Amit and Shalev-Shwartz, Shai},
  booktitle={Conference on Learning Theory},
  pages={287--316},
  year={2014},
  organization={PMLR}
}

@incollection{natarajan1988two,
  title={Two new frameworks for learning},
  author={Natarajan, Balas K and Tadepalli, Prasad},
  booktitle={Machine Learning Proceedings 1988},
  pages={402--415},
  year={1988},
  publisher={Elsevier}
}

@inproceedings{NIPS2016_59f51fd6,
 author = {David, Ofir and Moran, Shay and Yehudayoff, Amir},
 booktitle = {Advances in Neural Information Processing Systems 29},
 editor = {D. Lee and M. Sugiyama and U. Luxburg and I. Guyon and R. Garnett},
 title = {Supervised learning through the lens of compression},
 year = {2016}
}

@INPROCEEDINGS {10353103,
author = {I. Aden-Ali and Y. Cherapanamjeri and A. Shetty and N. Zhivotovskiy},
booktitle = {2023 IEEE 64th Annual Symposium on Foundations of Computer Science (FOCS)},
title = {Optimal PAC Bounds without Uniform Convergence},
year = {2023},
volume = {},
issn = {},
pages = {1203-1223},
doi = {10.1109/FOCS57990.2023.00071},
url = {https://doi.ieeecomputersociety.org/10.1109/FOCS57990.2023.00071},
publisher = {IEEE Computer Society},
address = {Los Alamitos, CA, USA},
month = {nov}
}

@inproceedings{10.1145/3564246.3585190,
author = {Charikar, Moses and Pabbaraju, Chirag},
title = {A Characterization of List Learnability},
year = {2023},
isbn = {9781450399135},
publisher = {Association for Computing Machinery},
address = {New York, NY, USA},
url = {https://doi.org/10.1145/3564246.3585190},
doi = {10.1145/3564246.3585190},
booktitle = {Proceedings of the 55th Annual ACM Symposium on Theory of Computing},
pages = {1713–1726},
numpages = {14},
location = {Orlando, FL, USA},
series = {STOC 2023}
}

@inproceedings{NIPS2006_a11ce019,
 author = {Rubinstein, Benjamin and Bartlett, Peter and Rubinstein, J.},
 booktitle = {Advances in Neural Information Processing Systems},
 editor = {B. Sch\"{o}lkopf and J. Platt and T. Hoffman},
 pages = {},
 publisher = {MIT Press},
 title = {Shifting, One-Inclusion Mistake Bounds and Tight Multiclass Expected Risk Bounds},
 url = {https://proceedings.neurips.cc/paper_files/paper/2006/file/a11ce019e96a4c60832eadd755a17a58-Paper.pdf},
 volume = {19},
 year = {2006}
}

@article{Neumann1928,
author = {Neumann, J. von},
journal = {Mathematische Annalen},
pages = {295-320},
title = {Zur Theorie der Gesellschaftsspiele},
url = {http://eudml.org/doc/159291},
volume = {100},
year = {1928},
}

@INPROCEEDINGS {partial_concept,
author = {N. Alon and S. Hanneke and R. Holzman and S. Moran},
booktitle = {2021 IEEE 62nd Annual Symposium on Foundations of Computer Science (FOCS)},
title = {A Theory of PAC Learnability of Partial Concept Classes},
year = {2022},
volume = {},
issn = {},
pages = {658-671},
doi = {10.1109/FOCS52979.2021.00070},
url = {https://doi.ieeecomputersociety.org/10.1109/FOCS52979.2021.00070},
publisher = {IEEE Computer Society},
address = {Los Alamitos, CA, USA},
month = {feb}
}

@inproceedings{
hanneke2024improved,
title={Improved Sample Complexity for Multiclass {PAC} Learning},
author={Steve Hanneke and Shay Moran and Qian Zhang},
booktitle={The Thirty-eighth Annual Conference on Neural Information Processing Systems},
year={2024},
url={https://openreview.net/forum?id=l2yvtrz3On}
}

@book{lattimore2020bandit,
  title={Bandit algorithms},
  author={Lattimore, Tor and Szepesv{\'a}ri, Csaba},
  year={2020},
  publisher={Cambridge University Press}
}

@article{freund1995boosting,
  title={Boosting a weak learning algorithm by majority},
  author={Freund, Yoav},
  journal={Information and computation},
  volume={121},
  number={2},
  pages={256--285},
  year={1995},
  publisher={Elsevier}
}

@article{cesa2004generalization,
  title={On the generalization ability of on-line learning algorithms},
  author={Cesa-Bianchi, Nicolo and Conconi, Alex and Gentile, Claudio},
  journal={IEEE Transactions on Information Theory},
  volume={50},
  number={9},
  pages={2050--2057},
  year={2004},
  publisher={IEEE}
}

@inproceedings{hanneke:25a,
author = {S. Hanneke and Q. Meng and A. Shaeiri},
title = {Representation Preserving Multiclass Agnostic to Realizable Reduction},
booktitle = {Proceedings of the $42^{\mathrm{nd}}$ International Conference on Machine Learning},
year = 2025
}

@inproceedings{hanneke:25b,
author = {S. Hanneke},
title = {Agnostic Active Learning Is Always Better Than Passive Learning},
booktitle = {Advances in Neural Information Processing Systems 38},
year = 2025
}

@inproceedings{erez:24,
author = {L. Erez and A. Peled-Cohen and T. Koren and Y. Mansour and S. Moran},
title = {Fast Rates for Bandit {PAC} Multiclass Classification},
booktitle = {Advances in Neural Information Processing Systems 37},
year = 2024 
}

@inproceedings{attias:23,
author = {I. Attias and S. Hanneke and A. Kalavasis and A. Karbasi and G. Velegkas},
title = {Optimal Learners for Realizable Regression: {PAC} Learning and Online Learning},
booktitle = {Advances in Neural Information Processing Systems 36},
year = 2023
}

@inproceedings{carmon:24,
author = {D. Carmon and R. Livni and A. Yehudayoff},
title = {The Sample Complexity of {ERMs} in Stochastic Convex Optimization},
booktitle = {Proceedings of the $27^{\mathrm{th}}$ International Conference on Artificial Intelligence and Statistics},
year = 2024
}

@inproceedings{feldman:16,
author = {V. Feldman},
title = {Generalization of {ERM} in Stochastic Convex Optimization: {T}he Dimension Strikes Back},
booktitle = {Advances in Neural Information Processing Systems 29},
year = 2016
}

\end{document}